\documentclass{article}
\usepackage{arxiv}

\usepackage[utf8]{inputenc} 
\usepackage[T1]{fontenc}    
\usepackage{hyperref}       
\usepackage{url}            
\usepackage{booktabs}       
\usepackage{amsfonts}       
\usepackage{nicefrac}       
\usepackage{microtype}      
\usepackage{lipsum}		
\usepackage{graphicx}
\usepackage{natbib}
\usepackage{doi}
\usepackage{placeins} 

\usepackage{amsmath, amssymb}
\usepackage{yfonts}
\usepackage{commath}
\usepackage{longtable}
\usepackage{subcaption}

\usepackage{float}
\usepackage{amsthm}
\usepackage{dsfont}
\usepackage{xcolor, tikz}
\usepackage{hyperref}
\usepackage{fontawesome}
\usepackage{mathtools}
\usepackage{cleveref}
\usepackage{mathabx}
\usepackage{array, makecell}
\usepackage{footnote}
\makesavenoteenv{tabular}
\makesavenoteenv{table}

 \newcommand{\reals}{\mathbb{R}}
  
  \newcommand{\Ex}{\mathbb{E}}
  \renewcommand{\Pr}{\mathbb{P}}
  \newcommand{\X}{{X}}
  \newcommand{\Y}{{Y}}
  \newcommand{\F}{\mathcal{F}}
  
   \newcommand{\Lo}[1]{{\mathcal L_{#1}}}
   \newcommand{\bLo}[1]{{\mathcal L^{0/1}_{#1}}}
   \newcommand{\rmse}[1]{{\mathcal L^{\mathrm{RMSE}}_{#1}}}
   \newcommand{\bopt}[1]{{\mathrm{opt}^{0/1}_{#1}}}
   
   \newcommand{\lo}{\ell}
   \newcommand{\blo}{\ell^{0/1}}
  \newcommand{\ce}[1]{{\mathrm{CE}_{#1}}}
  \newcommand{\pc}[1]{{\mathrm{PC}_{#1}}}
   \newcommand{\indct}[1]{\mathds{1}\left[{#1}\right]}
  \DeclareMathOperator*{\E}{\mathbb{E}}
  \newcommand{\supp}[1]{\mathrm{supp}(#1)}
  \newcommand{\range}[1]{\mathrm{range}_{#1}}
 \newcommand{\auc}{\mathrm{AUC}}

\newtheorem{definition}{Definition}[section]
\newtheorem{theorem}{Theorem}
\newtheorem{observation}[theorem]{Observation}
\newtheorem{corollary}{Corollary}
\newtheorem{example}{Example}

\title{Calibration through the Lens of Interpretability}

 \date{}	

\author{Alireza Torabian\\
    EECS Department\\
	York University\\
    Toronto, Canada\\
	\texttt{talireza@yorku.ca} \\
	\And
	\href{https://orcid.org/0000-0003-1179-7134}{\includegraphics[scale=0.06]{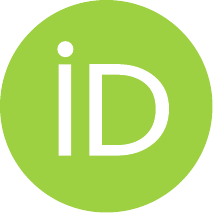}\hspace{1mm}Ruth Urner} \\
    EECS Department\\
	York University\\
    Toronto, Canada\\
	\texttt{uruth@yorku.ca} \\
}

\renewcommand{\headeright}{}
\renewcommand{\undertitle}{}
\renewcommand{\shorttitle}{Calibration through the Lens of Interpretability}

\hypersetup{
pdftitle={InterpretableCalibration},
pdfsubject={q-bio.NC, q-bio.QM},
pdfauthor={Alireza Torabian, Ruth Urner},
pdfkeywords={Calibration, Axiomatic analysis, Evaluation measures},
}

\begin{document}
\maketitle

\begin{abstract}
	Calibration is a frequently invoked concept when useful label probability estimates are required on top of classification accuracy. A calibrated model is a function whose values correctly reflect underlying label probabilities. Calibration in itself however does not imply classification accuracy, nor human interpretable estimates, nor is it straightforward to verify calibration from finite data. There is a plethora of evaluation metrics (and loss functions) that each assess a specific aspect of a calibration model. In this work, we initiate an axiomatic study of the notion of calibration. We catalogue desirable properties of calibrated models as well as corresponding evaluation metrics and analyze their feasibility and correspondences. We complement this analysis with an empirical evaluation, comparing common calibration methods to employing a simple, interpretable decision tree.
\end{abstract}


\section{Introduction}\label{sec:intro}
In many applications it is important that a classification model not only has high accuracy but that a user is also provided with a reliable estimate of confidence in the predicted label. Calibration is a concept that is often invoked to provide such confidence estimates to a  user. As such, calibration is a notion that is \emph{inherently aimed at human interpretation}. In binary classification, a perfectly calibrated model $f$ provides the guarantee that if it predicts $f(x) = p$ on some instance $x$, then \emph{among the set of all instances on which $f$ assigns this value $p$} the probability of label $1$ is indeed $p$ (and the probability of label $0$ thus $1-p$).

While calibration is generally considered useful, we would argue that in many cases, even if achieved, it is doomed to fail at its original goal of providing insight to a human user: for most suitably complex classification models, a human user that observes $f(x) = p$ has no notion of the set of all instances on which $f$ also outputs $p$. The promise given by calibration is thus meaningless.

In this work we investigate which (additional) properties would actually yield human understandable calibration scores. We take an axiomatic (or property based) approach and start by outlining formal desiderata for a calibrated model. The first and obvious property is high classification accuracy, which is not implied by calibration. Second, a property that is often implicitly aimed at in the context of calibration is that the predictor actually approximates the regression function $\eta(x) = \Pr[y=1|x]$ of the data-generating process  \cite{SunSH23}. This is however also not implied by calibration. We then propose three properties that relate to interpretability: 1) that the pre-images  (or level sets) $f^{-1}(r)$ of the model are identifiable to a user, 2) that the range of values that the model outputs is not too large, and 3) that the model is monotonic with respect to the underlying regression function. Section \ref{sec:properties} starts with a detailed discussion and motivation for these suggested desiderata. In that section we also formally analyze the interplay between these (initially strictly phrased) properties.

Since a learned model can usually not be expected to satisfy properties such as optimal accuracy or calibration perfectly, in Section \ref{sec:properties_relaxation} we then move to outlining relaxations of our desiderata in form of measures of distance from the properties. The discussion and analysis in that section focuses on measures at the population level of a data-generating process. We analyze how simple operations on a predictor, which may improve its calibration, affect these measures.

In the last Section \ref{sec:experiments}, we  deal with empirical, finite data based versions of these measures. We again start by outlining and discussing these empirical measures, most of which are from the literature on calibration. Our experiments on a variety of real world datasets then evaluate them on a simple, inherently interpretable model for calibration, namely a decision tree, and compare its performance to three other, not necessarily interpretable standard calibration methods. The goal of this section is to take a model for which we can control our interpretability criteria (the pre-images $f^{-1}(r)$ are here the leafs of the tree and thus interpretable, and the number of these can be set by a user), and assess how this interpretable model compares to other methods in terms of other performance measures.

In summary, we systematically outline and analyze the interplay between desirable properties, evaluation measures and sometimes implicit objectives on three levels: on an idealized level as axioms (or deterministic properties), on the distributional level as probabilistic metrics, and on the empirical level as measures to be estimated from finite data. 
While the first level is aimed at capturing the aleatoric uncertainty in the data generation, and the second defines measures of how well a predictor reflects this uncertainty, the last level integrates the epistemic uncertainty, namely how to estimate these qualities from samples.
Our work sheds light on the role of interpretability in the context of calibration, which we find essential for calibration to be meaningful and thus useful to users.

\subsubsection*{Overview on related work}
Calibration is a well established notion with studies on this concept dating back decades \cite{Dawid1982TheWB,FosterVohra98,KakadeF08}. Summarizing this rich body of literature is beyond the scope of this manuscript, but recent surveys provide an overview on the concept of calibration, common methods aimed at achieving it and popular evaluation metrics \cite{Filho2023,wang2024calibration}. 
With the advent of increasingly powerful yet opaque machine learning models, the concept of calibration has enjoyed renewed interest and research activity in recent years \cite{blasiok2022unifying,MBCT,BlasiokGHN23,FamigliniCC23}. 

Methods to obtain calibration broadly fall into two categories: post-processing an existing model or directly training in a way that promotes calibration in the learned model. Platt Scaling (PS) \cite{plattscaling} and Isotonic Regression (IS) \cite{isotonic} (which we include in our experiments) are two well established methods in the former category. Another class of commonly used post-processing methods, for which formal guarantees also exist, is re-calibration based on binning \cite{naeini2015bayesian,scaling_binning,metric_validity_auc,SunSH23}. To directly promote calibration, training by optimizing a \emph{proper loss} is often recommended. Proper losses are minimized by the data-generating distribution's regression function. Very recent work has analyzed when this actually leads to calibrated models \cite{BlasiokGHN23}.

A major challenge with understanding how to obtain calibrated models is the lack of clear, commonly accepted criteria for ``how uncalibrated'' a model is. There are a variety of studies that aim to address this inherent ambiguity both from a practical point of view, by systematically developing and comparing evaluation methods \cite{PosoccoB21,Filho2023,FamigliniCC23} and from a theoretical perspective by formally establishing failure modes and success guarantees \cite{blasiok2022unifying}. 

While some recent studies point out contributions of calibration \emph{for model interpretability} \cite{ScafartoPB22}, we are not aware of a systematic analysis of the interpretability \emph{of calibration} itself, which is the focus of this work.

\section{Formal Setup}\label{sec:setup}
\paragraph{Binary Classification} We consider the standard setup of statistical learning: We let $\X$ denote a feature space and $\Y = \{0,1\}$ the label space. The data generation is modelled as a distribution $D$ over $\X\times \Y$. We use $D_\X$ to denote the marginal of $D$ over $\X$. We use $\supp{\cdot}$  to denote the support of  a distribution. With slight abuse of notation, for a distribution $D$ over $\X\times\Y$,
we will often write $\supp{D}$ to also refer to the support $\supp{D_X}$ of the marginal $D_\X$. 
Further, we let $\eta_D:\X\to[0,1]$ denote the \emph{regression function} of the distribution $D$:
\[
\eta_D(x) = \Pr_{(x', y)\sim D}[y=1 | x' =x]
\]
A \emph{predictor} is a function $f:\X\to\reals$ that assigns every instance a real valued score. 
Given a data generating distribution $D$, we let $\range{D}(f)$ denote the 
\emph{effective range} of the predictor, namely the smallest set $R$ such that with probability $1$ over drawing $x\sim D_X$, we have $f(x)\in R$. For discrete distributions, we can alternatively define  $\range{D}(f) \coloneqq \{f(x) | x \in \supp{D_X}\}$. 
For simplicity, we will usually use  statements such as ``for all $x\in\supp{D_X}$'' instead of ``with probability $1$ over $D_X$'', and ``there exist $x\in\supp{D_X}$'' instead of ``with probability greater than $0$ over $D_X$''. These concepts are equivalent for discrete distributions (and under some continuity assumptions on functions in the non-discrete case). The above substitutions can be made for more general cases.

We define the \emph{cells generated by predictor $f$} as the subsets of $X$ on which $f$ is constant, i.e., the pre-images under $f$ of the values in $\range{D}(f)$; a predictor $f$ thus partitions  $X$ into cells.

A \emph{classifier} is a function $h: \X\to \Y$ that assigns every feature vector a class label. For binary classification, it is common to threshold some predictor for this. Given $f: \X\to\reals$, we define the classifier induced by $f$ with threshold $\theta$ as
\[
 f_\theta(x) = \indct{f(x) \geq \theta}
\]
where $\indct{\cdot}$ denotes the indicator function.
We use $\F = \reals^\X$ to denote the set of all (measurable) predictors.
Predictors $f$ are evaluated by means of a \emph{loss function} $\lo:\F\times\X\times\Y\to\reals$, where $\lo(f,x,y)$ indicates the quality of prediction $f(x)$  given observed label $y$. The goal is to achieve low \emph{expected loss} 
\[
\Lo{D}(f) = \Ex_{(x,y)\sim D}[\lo(f,x,y)].
\]
The \emph{binary loss} (or \emph{$0/1$-loss}) is the standard evaluation metric for classifiers
$\blo(h,x,y) = \indct{h(x)\neq y}$.
The \emph{Bayes classifier} is a classifier with minimal expected binary loss, denoted by $\bopt{D}$, the \emph{Bayes loss}  of $D$. 

\paragraph{Calibration}
In many applications, it is desirable to not only achieve low classification loss (that is high accuracy), but to have a predictor that accurately reflects \emph{probabilities} of the label events. The notion of \emph{calibration} defines such a property; namely, that the predicted value $f(x)$ accurately reflects the probability of seeing label $1$ among all instances that are given value $f(x)$ \cite{Dawid1982TheWB,FosterVohra98,KakadeF08,Filho2023}.
\begin{definition}\label{def:calibration}
A predictor $f:\X\to[0,1]$ is \emph{calibrated} if for all $x\in\X$ we have:
\[
f(x) = \Ex_{(x',y')\sim D}[y' ~\mid~ f(x') = f(x)]
\]
\end{definition}
Predictors are rarely  expected to be perfectly calibrated as in the above definition. There are a variety of notions to measure ``degrees of miscalibration'' both with respect to the underlying distribution and empirically as observed on a dataset. We outline some of these in later parts of this work (see Sections \ref{sec:properties_relaxation} and \ref{sec:experiments}). We note here that, due to the conditioning on the level sets of the predictor in the definition of calibration, there is no straightforward way of measuring miscalibration, in particular not as an expectation over a pointwise defined loss function which would depend only on a predictor $f$ and an observation $(x,y)$.

\section{Desiderata for Calibration}\label{sec:properties}
We now list some formal requirements for predictors that are aimed to be calibrated. 
The goal here is to make often implicit motivations explicit and formal.
The first, obvious, requirement (first item in the list) is calibration itself as defined in Definition \ref{def:calibration}.
However, such a predictor should often have additional qualities that are not subsumed by the notion of calibration.
Of course it is still desirable (if not imperative) that the predictor allows to be thresholded into a classifier with high accuracy (second item in the list). Moreover, the hope behind calibration is often that the predictor $f$ will actually be a good representation of data-generating distribution's regression function $\eta_D$ (third item). Neither of these latter two requirements is implied by calibration.
We will formally elaborate on this in Section \ref{ss:deterministic_analysis} below.

Furthermore, we would argue that the concept of calibration is inherently aimed at \emph{aiding human interpretation}. The intent of providing a probability estimate rather than simply outputting the most likely label is to provide a human user with a better way to gauge the certainty with which the user should expect to see a certain label. However for this estimate to be meaningful to a human user, \emph{the user needs to have a notion of the pool of instances that also received this particular estimate}. That is, if a calibrated model outputs $f(x) = 0.7$, the human user needs a notion of the set $f^{-1}(0.7) = \{x \in X | f(x) = 0.7\}$, among which this user is now promised that $70\%$ of instances will have label $1$. Note that in this case calibration does not imply that $70\%$ of instances with this exact (or similar) feature vector $x$ will have label $1$. Thus, the mere statement $f(x) = 0.7$ (even from a calibrated predictor) does not provide insight into the data generating process. The fourth item in the list below captures these considerations: the cells induced by a calibrated predictor should be interpretable to a human user and there shouldn't be too many of such cells.

The fifth and last item in the list of requirements below also aims at human interpretability. 
If a user observes the predicted values on two input instances $f(x)$ and $f(x')$, say $f(x) = 0.57$ and $f(x') = 0.89$, the most meaningful insight might be that the first instance $x$ is less likely to have  label $1$ than the second instance $x'$, (based on observing that $f(x) < f(x')$). The exact values ($0.57$ and $0.89$) may not be as easy to make sense of.
However, this type of pairwise comparison is valid only if the predictor is point-wise monotonic with respect to the data-generating distribution's regression function.

\paragraph{Formal requirements} 
We let $f:\X\to\reals$ denote a predictor and $D$ be a distribution over $\X\times\Y$. The list below summarizes our desiderata for $f$:
\begin{enumerate}
\item \textbf{Calibration.} $f$ is perfectly calibrated (see Definition \ref{def:calibration}):
\[\forall r \in \range{D}(f) ~:~ \mathop{\mathbb{E}_{(x, y)\sim D}} [y~|~f(x)=r] = r.\]
\item \textbf{Classification accuracy.} Thresholding on $f$ yields an optimal classifier:
\[
\exists \theta\in\reals ~:~ \bLo{D}(f_\theta) = \bopt{D}.
\] 
 \item \textbf{Approximating the regression function}: 
$f$ perfectly approximates $\eta_D$:
$$\forall x\in\supp{D}~:~ f(x) = \eta_D(x).$$
\item \textbf{Interpretability}: The cells induced by $f$, that is the pre-images $f^{-1}(r)\coloneqq\{x \in \X ~\mid~ f(x)=r\}$ for $r\in\range{D}(f)$, are meaningful to a human user. 
Moreover, there are relatively few induced cells. That is $|\{f^{-1}(r) ~|~r \in \range{D}(f)\}| = |\range{D}(f)|$ is small.

\item \textbf{Monotonicity}: Predictor $f$ generates  probability estimates that are monotonic with respect to the regression function $\eta_D$, that is:
\begin{align*}
\forall x_i, x_j\in\supp{D} ~:~ (\eta_D(x_i) - \eta_D(x_j))\cdot (f(x_i) - f(x_j)) \geq 0
\end{align*}
If equality holds only when $\eta_D(x_i) = \eta_D(x_j)$, we call $f$ \emph{strictly monotonic} with respect to $\eta_{D}$.
\end{enumerate}
We will start by analyzing these strictly phrased properties. In Section \ref{sec:properties_relaxation} below, we will introduce and investigate probabilistic relaxations of these properties.

\subsection{Interplay of Strict Properties}\label{ss:deterministic_analysis}
We start our analysis by investigating relationships, implications and compatibilities between the above desiderata. 
At first glance, it might appear as if calibration is a stronger requirement than the existence of a threshold for optimally accurate classification. However, it is not difficult to see, and generally known \cite{Filho2023}, that calibration is actually a property that is independent of accuracy.
A predictor can be perfectly calibrated while effectively useless for classification. And conversely a predictor can be highly accurate while not being calibrated at all.

\begin{observation}
Calibration does not imply optimal classification accuracy and optimal classification accuracy does not imply calibration.
\end{observation}
\begin{proof}
Consider a one-dimensional feature space, $\X = \reals$, and a distribution $D$ that has marginal mass distributed uniformly on two points,  $D_X(-1) = D_X(+1) = 0.5$,  with a deterministic regression function $\eta_D(x) = \indct{x\geq 0}$. Now the constant predictor $f(x) = 0.5$ is perfectly calibrated, but any threshold $\theta\in\reals$ will result in worst possible classification loss $\bLo{D}(f_\theta)= 0.5$. On the other hand, a predictor $g$ with $g(x) = 0.5-\epsilon$ for $x <0$ and  $g(x) = 0.5+\epsilon$ for $x \geq 0$ for any $\epsilon >0$ admits a threshold (namely $\theta  = 0.5$) such that the resulting classifier $g_\theta$ has perfect classification loss $\bLo{D}(g_\theta) = 0$ while not being calibrated.
\qed\end{proof}

Of course, the regression function $\eta_D$ is always a predictor (albeit usually an unknown one) that is both perfectly calibrated and optimally accurate (by definition, with threshold $\theta = 0.5$). However, we now show that in most cases (except for distributions where the regression function is overly simple) it is not the only predictor that enjoys these two qualities. This then means that these two properties together (calibration and possibility of optimal classification accuracy) do not imply that the regression function $\eta_D$ is well approximated.

\begin{theorem}
There exist predictors $f$ different from $\eta_D$ (with positive probability)
satisfying both perfect calibration and optimal classification accuracy if and only if one of the sets $(\range{D}(\eta_D) \cap [0, 0.5))$ and $(\range{D}(\eta_D) \cap [0.5,1])$ has size at least $2$ (that is, if and only if a Bayes optimal predictor outputs both labels and the effective range of $\eta_D$ has size at least $3$; or a Bayes optimal predictor outputs only one label and the effective range of $\eta_D$ has size at least $2$).
\end{theorem}

\begin{proof}
Let's assume that at least one of the sets $\range{D}(\eta_D) \cap [0, 0.5)$ and $\range{D}(\eta_D) \cap [0.5,1]$ has size at least $2$.
Without loss of generality we can assume that 
there exist $\eta_1, \eta_2 \in\range{D}(\eta_D)$, with $\eta_1, \eta_2 < 0.5$ and $\eta_1\neq \eta_2$.
Let's denote regions where the regression function takes on these values by $\X_1 = \eta_D^{-1}(\eta_1) \subseteq X$, and $\X_2 = \eta_D^{-1}(\eta_2)\subseteq X$. By definition of the effective range, these sets have positive probability under $D_X$. Now consider the predictor 
{\small
\[
f(x) = 
\begin{cases}
\E_{(x',y) \sim D}[y \mid x'\in (\X_1\cup X_2)] & \text{ if } x \in (\X_1\cup X_2)\\
\eta_D(x) & \text{ if } x \notin (\X_1\cup\X_2).
\end{cases}
\]
}
By construction, this predictor, thresholded at $0.5$ has the same classification loss as $\eta_D$ (namely $\bopt{D}$) while being different from $\eta_D$ with positive probability.

Conversely, assume that there exists a predictor $f$ that is perfectly calibrated and achieves Bayes loss with some threshold $\theta\in[0,1]$, but is not identical to $\eta_D$ (meaning the functions differ with positive probability with respect to $D_\X$). Since $\bLo{D}(f_\theta) = \bopt{D}$, the sets $f^{-1}([0, \theta))\cap \supp{D}$ and $\eta_D^{-1}([0, 0.5))\cap \supp{D}$ must be identical and the sets $f^{-1}([\theta,1])\cap \supp{D}$  and $\eta_D^{-1}([0.5,1])\cap \supp{D}$  must be identical. Now if $\eta_D$ was constant on both of these sets, then the only way for $f$ to be calibrated would be to also take on that same constant values (and thus $f$ would be identical to $\eta_D$). Thus, if $f$ differs from $\eta_D$ in the support of $D_\X$ while being calibrated, then $\eta_D$ is not constant on at least one of $\eta_D^{-1}([0.5,1])\cap \supp{D}$ or $\eta_D^{-1}([0, 0.5))\cap \supp{D}$, which implies that at least one of $\range{D}(\eta_D) \cap [0, 0.5)$ and $\range{D}(\eta_D) \cap [0.5,1]$ has size at least $2$.
\qed\end{proof}

\begin{corollary}
Perfect calibration and optimal classification accuracy together do not imply perfect approximation of $\eta_D$.
\end{corollary}

Rather than calibration, strict monotonicity is a property that is closely related to both optimal classification accuracy and approximation of $\eta_D$.

\begin{observation}
Strict monotonicity implies optimal classification accuracy.
\end{observation}
\begin{proof}
Consider a predictor $f$ and assume that $f$ satisfies strict monotonicity with respect to $\eta_D$. Using the threshold $0.5$ on the regression function $\eta_D$, we can split the set $\supp{D}$ into two disjoint subsets $X_- \coloneqq \{x \in \supp{D}~:~ \eta_D(x) < 0.5\}$ and $X_+= \{x \in \supp{D}: \eta_D(x) \geq 0.5\}$.
Let $f_{X_-} \coloneqq \{f(x) : x \in X_-\}$ and $f_{X_+} \coloneqq \{f(x) : x \in X_+\}$ be the ranges of values that $f$ takes on $X_-$ and $X_+$ respectively. 
For any $x_i$ from $X_-$ and any $x_j$ from $X_+$, $f(x_i) < f(x_j)$ since $\eta_D(x_i) < \eta_D(x_j)$ and $f$ is strictly monotonic. This shows that any member of $f_{X_-}$ is smaller than any member of $f_{X_+}$. Therefore, $\inf(f_{X_+}) \geq \sup(f_{X_-})$. Thus $\theta = (\inf(f_{X_+}) + \sup(f_{X_-})) / 2$ is a threshold on $f$ that achieves Bayes loss.
\qed\end{proof}

\begin{theorem}
A predictor $f$ perfectly approximates the regression function $\eta_D$ if and only if it is perfectly calibrated and strictly monotonic with respect to $\eta_D$.
\end{theorem}

\begin{proof}
If $f$ perfectly approximates $\eta_D$ (that is, the are identical with probability $1$ over $D_X$), then $f$ is obviously strictly monotonic and perfectly calibrated.

Now we will argue that a predictor that is strictly monotonic with respect to $\eta_D$ and calibrated, also perfectly approximates $\eta_D$. Equivalently, we show that if $f$ is strictly monotonic, but does not perfectly approximate $\eta_D$, then $f$ is not calibrated.
   So let's assume $f$ is strictly monotonic but not perfectly approximating $\eta_D$. Then there exists an $x' \in \supp{D}$ with $ f(x') \neq \eta_D(x')$. Let $S \coloneqq \{x \in \supp{D}: f(x) = f(x')\}$. Since $f$ is strictly monotonic, we have $ \eta_D(x) = \eta_D(x')$ for all $x\in S$. Therefore, 
    \[
    \E_{(x, y)\sim D} [y~|~f(x)=f(x')] =  \E_{(x, y)\sim D} [y~|~x \in S] = \eta_D(x').
    \]
    Since $\eta_D(x') \neq f(x')$, $\Ex_{(x, y)\sim D} [y~|~f(x)=f(x')] \neq f(x')$, thus $f$ is not calibrated. 
\qed\end{proof}

\begin{corollary}
Neither calibration nor strict monotonicity alone implies perfect approximation of $\eta_D$.
\end{corollary}
In Figure \ref{fig:properties_venn} below, we illustrate the relationship between the properties.
\begin{figure}[h]
         \begin{center}
        \includegraphics[width=.5\textwidth]{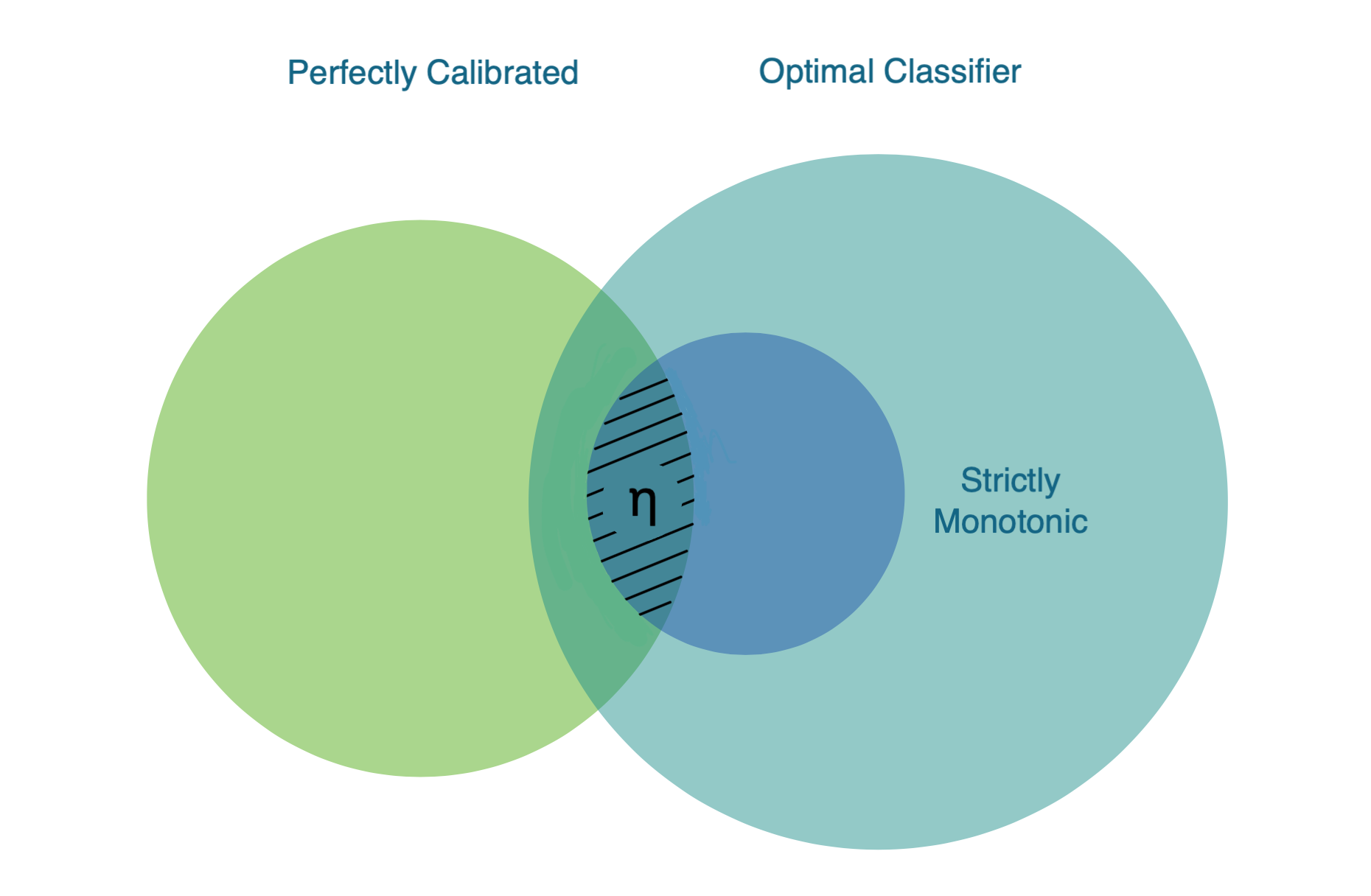}
        \caption[Interplay of calibration properties]{Interplay of calibration desiderata. The intersection of strictly monotonic and perfectly calibrated predictors only contains the regression function $\eta_D$ (and functions that agree with $\eta_D$ with probability $1$ over $D_X$).}
        \label{fig:properties_venn}
        \end{center}
    \end{figure}

\section{Relaxed Desiderata for Calibration} \label{sec:properties_relaxation}
A predictor $f$ that is learned from finite samples is unlikely to fulfill the desiderata precisely. Thus, we now outline relaxed, probabilistic  versions of our five desirable properties, or measures of how much the properties are violated. We here focus on population level measures (rather than possible estimates from finite samples, some of which we discuss in Section \ref{sec:experiments}). It is important to note that the interpretability of the pre-images of a predictor is not a property that is quantifiable by means of a mathematical definition. Therefore, we focus in this section on quantifying the size of the effective range as an aspect of interpretability, and propose a novel, distribution based measure for this.

\begin{enumerate}
\item \textbf{Calibration}: Degree of calibration is measured by the \emph{$L_p$-norm expected calibration error} \cite{metric_validity_auc}:
\[
\ce{p, D}(f)=\big(\E_{(x, y)\sim D}[|f(x) - \E_{(x', y')\sim D}[y'~|~f(x')=f(x)]|^p]~\big)^{1/p}
\]
\item \textbf{Classification accuracy}: 
The quality of classification is measured by the standard expected classification loss for a thresholded predictor $f_\theta$:
\[
\bLo{D}(f_\theta) = \Ex_{(x,y)\sim D}\indct{f_\theta(x) \neq y}
\]

\item \textbf{Approximating the regression function}: To assess whether the predictor $f$ is effectively approximating the regression function, we use the Mean Squared Error (MSE) \cite{Brier1950,SunSH23} :
\[
\mathrm{MSE}_{D}(f) = \Ex_{(x,y)\sim D}[(y - f(x))^2]
\]
Note that MSE is the expectation over a \emph{proper loss}, namely the quadratic loss $\lo^2(f, x,y) = (f(x) - y)^2$, and is thus minimized (over all predictors) by the regression function $\eta_D$ \cite{BlasiokGHN23}. It can thus be viewed as a distance from $\eta_D$.
\item \textbf{Interpretability}: To relax the strict measure of counting cells induced by $f$ (ie.\ $|\range{D}(f)|$ being small), we introduce the \emph{Probabilistic Count (PC)}, as a novel measure that quantifies the size of a predictor's range, taking into account the data-generating distribution. 
    For distribution $D$ over $X \times Y$, we define the \emph{probabilistic count} of predictor $f: X \rightarrow \mathop{\mathbb{R}}$  with respect to $D$ as:
    \[
    \pc{D}(f) = \frac{1}{
    \Pr_{x, x'\sim D_X} [f(x) = f(x')]
    }.
    \]
We show that $\pc{D}(f)$ is always at most $|\range{D}(f)|$. Appendix Section \ref{sec:probabilistic_count} contains this result (Theorem \ref{thm:pc_upperbound}) and further illustrations of this measure.
\item \textbf{Monotonicity}: Kendall's $\tau$ (tau) coefficient is a measure of the monotonicity of finite samples \cite{kendallsTauPuka2011}. For any set of samples $(x_1, y_1), ..., (x_n, y_n)$, any pair of samples $(x_i, y_i)$ and $(x_j, y_j)$, where $i<j$, are discordant if $(x_i - x_j) . (y_i - y_j) < 0$. Kendall's Tau coefficient is defined as:
\[
\tau = 1 - \frac{2 \times \text{number of discordant pairs}}{{n \choose 2}}
\]
Kendall's $\tau$ coefficient is in the range $-1 \leq \tau \leq 1$. $\tau = 1$ represents perfect agreement between the ranking of two variables, and $\tau = -1$ represents perfect disagreement, i.e. one ranking is the reverse of the other. This coefficient can be used to measure monotonicity, but it doesn't consider the ties to measure strict monotonicity.
We now introduce a probabilistic version of this coefficient to measure the monotonicity of two random variables, namely $\eta_D(x)$ and $f(x)$.
We define the \emph{probabilistic Kendall's Tau coefficient} for predictor $f: X \rightarrow \mathop{\mathbb{R}}$ with respect to the 
distribution $D$ over $X \times Y$ as:
\begin{align}
\mathrm{KT}_{D}(f) &= 1 - 2 \times \mathop{\mathbb{P}_{x, x'\sim D_X}}[(\eta_D(x) - \eta_D(x'))\cdot (f(x) - f(x')) < 0 ~\bigr\vert~ x \neq x']\nonumber
\end{align}
\end{enumerate}
In the remainder of this section we explore the effects of two intuitive operations that contribute to model interpretability: cell merging (with score averaging) and readjusting scores by label averaging in a cell. 
Through the analysis of these operations and their effects on the measures from the list above, we aim to clarify when a model can be simplified in a way that may increase its interpretability, without compromising other properties that are desirable for calibrated models. Table \ref{tab:merging_averaging_summary} summarizes the results of this section. The arrows indicate whether a measure can increase, decrease and remain the same through the operation.
\begin{table}[h]
\caption[Implications of cell merging and score averaging on the properties of a predictor]{Implications of cell merging and score averaging on the measures. The arrows and equality signs represent the possible outcomes for each measure. }
\centering
\renewcommand\cellset{\renewcommand\arraystretch{0.6}%
    \setlength\extrarowheight{0pt}}
\renewcommand{\arraystretch}{1.5}
\begin{tabular}{lc @{\hspace{1.2em}} c @{\hspace{1.2em}} c @{\hspace{1.2em}} c @{\hspace{1.2em}} c}
\hline
                                          & $\ce{p, D}$ & $\bLo{D}$ & $\mathrm{MSE}_{D}$ & $\pc{D}$ & $\mathrm{KT}_{D}$\\ \hline
\multicolumn{1}{l|}{\makecell{Cell merging along \\ with averaging scores}}         &
\makecell{$\downarrow=$ \\ (Thm. \ref{thm:join_on_ce})}        & \makecell{$\updownarrows=$ \\ (Obs. \ref{obs:join_on_closs})}       &   \makecell{$\updownarrows=$ \\ (Obs. \ref{obs:join_on_closs})}           &        \makecell{$\downarrow=$
\footnote{The conclusion for $\pc{D}$ holds for any new score for the cell according to Theorem \ref{thm:join_on_pc}.}
\\ (Thm. \ref{thm:join_on_pc})}
& \makecell{$\updownarrows=$ \\ (Obs. \ref{obs:join_on_closs})} \\  \hline  
\multicolumn{1}{l|}{Average label assigning} &  \makecell{$\downarrow=$ \\ (Thm. \ref{thm:average_on_ce_mse})} & \makecell{$\downarrow=$ \\ (Thm. \ref{thm:average_on_ce_mse})} & \makecell{$\downarrow=$ \\ (Thm. \ref{thm:average_on_ce_mse})} & \makecell{$\downarrow=$ \\ (Thm.\ref{thm:average_on_ce_mse})} & \makecell{$\updownarrows=$ \\ (Obs. \ref{obs:average_on_kt})} \\
\hline
\end{tabular}
\label{tab:merging_averaging_summary}
\end{table}

\subsection{Analysis of Cell Merging}\label{ss:cell_merg}
One aspect of interpretability of a predictor is the size of its effective range. When two cells are combined in that a joint value is assigned to all points from the two cells, the size of the effective range decreases. This can thus be viewed as a simple operation which will make the predictor more amenable to human interpretation.

\begin{definition}[Cell merge with score averaging]\label{def:cell_merge}
Let $D$ be a distribution over $X \times Y$, $f: X \rightarrow \mathop{\mathbb{R}}$ be a predictor and let $ r_1, r_2 \in \range{D}(f)$ be two values in the effective range of $f$. We say that predictor $g: X \rightarrow \mathop{\mathbb{R}}$ \emph{is obtained by $(r_1,r_2)$-cell merge of $f$} if $g$ satisfies:
\[
g(x) \coloneqq 
\begin{cases}
r & \text{if $f(x) = r_1$ or $f(x) = r_2$} \\
f(x) & \text{otherwise},
\end{cases}
\]
for some $r\in\reals$. We say that $g$ is obtained by $(r_1,r_2)$-cell merge of $f$ \emph{with score averaging} if $r=\frac{r_1 \cdot \mathop{\mathbb{P}_{x \sim D_X}}[f(x)=r_1] + r_2 \cdot \mathop{\mathbb{P}_{x \sim D_X}}[f(x)=r_2]}
{\mathop{\mathbb{P}_{x \sim D_X}}[f(x)=r_1] + \mathop{\mathbb{P}_{x \sim D_X}}[f(x)=r_2]}$.
\end{definition}
We start by showing that cell merging always decreases the probabilistic count (whether the score is averaged or an arbitrary new score is chosen).
\begin{theorem} \label{thm:join_on_pc}
Let $D$ be a distribution over $X \times Y$, $f,g: X \rightarrow \mathop{\mathbb{R}}$ be predictors and $ r_1, r_2 \in \range{D}(f)$. If $g$ is obtained by $(r_1, r_2)$-cell merge of $f$, then
\[
\pc{D}(g) \leq \pc{D}(f).
\]
\end{theorem}
\begin{proof} To prove $\pc{D}(g) \leq \pc{D}(f)$, it suffices to show that
\begin{align}
    \frac{1}{\pc{D}(g)} - \frac{1}{\pc{D}(f)} &= \mathop{\mathbb{P}_{x, x'\sim D_X}} [g(x) = g(x')] - \mathop{\mathbb{P}_{x, x'\sim D_X}} [f(x) = f(x')] \nonumber \\
    &= \mathop{\mathbb{E}_{x, x'\sim D_X}}[\indct{g(x) = g(x')} - \indct{f(x) = f(x')}] ~\geq~ 0  \nonumber
\end{align}
For any $x, x' \in \supp{D}$, $\indct{g(x) = g(x')}$ and $\indct{f(x) = f(x')}$ have the same value (either both are 1 or 0), except when both $x$ and $x'$ belong to cells where $f$ assigns values  $r_1, r_2$ or $r$, since these are the only different cells between $f$ and $g$. For all $x$ and $x'$ in these cells, $g(x) = g(x') = r$, thus $\indct{g(x) = g(x')} = 1$. Therefore we have $\indct{g(x) = g(x')} - \indct{f(x) = f(x')} \geq 0$ for all $x,x'\in\supp{D}$. 
\qed\end{proof}
We next investigate the effect of cell merging on the $L_p$ norm calibration error $\ce{p, D}$. We show $\ce{p, D}$ can only decrease if cells are merged with score averaging. 

\begin{theorem} \label{thm:join_on_ce}
Let $D$ be a distribution over $X \times Y$, $f,g: X \rightarrow \mathop{\mathbb{R}}$ be predictors and $ r_1, r_2 \in \range{D}(f)$. If $g$ is obtained by $(r_1, r_2)$-cell merge of $f$ with score averaging, then
we have
\[
\ce{p, D}(g) \leq \ce{p, D}(f).
\]
\end{theorem}
\begin{proof}
First note that for any predictor $h:X\to\reals$ and all $x\in X$, we have
 \begin{align}
 |h(x) - \mathop{\mathbb{E}_{(x', y')\sim D}}[y'~|~h(x')=h(x)]| \nonumber 
&= |h(x) - \mathop{\mathbb{E}_{x'\sim D_X}}[\eta_D(x')~|~h(x')=h(x)]| \nonumber \\
&= |\mathop{\mathbb{E}_{x'\sim D_X}}[h(x) - \eta_D(x')~|~h(x')=h(x)]| \nonumber \\
&= |\mathop{\mathbb{E}_{x'\sim D_X}}[h(x') - \eta_D(x')~|~h(x')=h(x)]| \nonumber
 \end{align}
Where the last inequality holds since the expectation is conditioned on any $x'$ such that $h(x')=h(x)$. Thus we get
\begin{equation}\label{eqn:CE_reform}
    \ce{p, D}(h) = \big(~\mathop{\mathbb{E}_{x\sim D_X}}[|\mathop{\mathbb{E}_{x'\sim D_X}}[h(x') - \eta_D(x')\mid h(x')=h(x)]|^p]~\big)^{1/p}
\end{equation}
for any predictor $h$.
Now note that if $r_1 = r_2$ then $f=g$, and $\ce{p, D}(f) = \ce{p, D}(g)$. Otherwise, according to Equation \ref{eqn:CE_reform} above applied to $f$ and $g$:
\begin{align}
      \ce{p, D}(f)^p - \ce{p, D}(g)^p &=  \mathop{\mathbb{E}_{x\sim D_X}}[|\mathop{\mathbb{E}_{x'\sim D_X}}[f(x') - \eta_D(x')\mid f(x')=f(x)]|^p] \nonumber \\
      &~~~- \mathop{\mathbb{E}_{x\sim D_X}}[|\mathop{\mathbb{E}_{x'\sim D_X}}[g(x') - \eta_D(x')\mid g(x')=g(x)]|^p] \nonumber \\
      &= \mathop{\mathbb{E}_{x\sim D_X}}[|\mathop{\mathbb{E}_{x'\sim D_X}}[f(x') - \eta_D(x')\mid f(x')=f(x)]|^p \nonumber \\
      &~÷~÷- |\mathop{\mathbb{E}_{x'\sim D_X}}[g(x') - \eta_D(x')\mid g(x')=g(x)]|^p] \nonumber
    \end{align}
  The cells that $f$ and $g$ have different expectations in the previous term are the ones with $f(x)$ equals to $r_1$ or $r_2$ or $r$. All members in these three cells are in the same cell in the range of $g$ with $g(x) = r$. So,
  \begin{align}
      \ce{p, D}(f)^p - \ce{p, D}(g)^p &=  |\mathop{\mathbb{E}_{x'\sim D_X}}[f(x') - \eta_D(x')\mid f(x')=r_1]|^p . \mathop{\mathbb{P}_{x\sim D_X}}[f(x)=r_1] \nonumber \\
      &~+ |\mathop{\mathbb{E}_{x'\sim D_X}}[f(x') - \eta_D(x')\mid f(x')=r_2]|^p . \mathop{\mathbb{P}_{x\sim D_X}}[f(x)=r_2] \nonumber \\
      &~+ |\mathop{\mathbb{E}_{x'\sim D_X}}[f(x') - \eta_D(x')\mid f(x')=r]|^p . \mathop{\mathbb{P}_{x\sim D_X}}[f(x)=r] \nonumber \\
      &~- |\mathop{\mathbb{E}_{x'\sim D_X}}[g(x') - \eta_D(x')\mid g(x')=r]|^p . \mathop{\mathbb{P}_{x\sim D_X}}[g(x)=r] \nonumber \\ \nonumber
  \end{align}
For the rest of the proof, we use the following notations:
\begin{align}
    e_1 &\vcentcolon= \mathop{\mathbb{E}_{x'\sim D_X}}[f(x') - \eta_D(x')\mid f(x')=r_1] \nonumber \\
    e_2 &\vcentcolon= \mathop{\mathbb{E}_{x'\sim D_X}}[f(x') - \eta_D(x')\mid f(x')=r_2] \nonumber \\
    e_3 &\vcentcolon= \mathop{\mathbb{E}_{x'\sim D_X}}[f(x') - \eta_D(x')\mid f(x')=r] \nonumber \\
    e' &\vcentcolon= \mathop{\mathbb{E}_{x'\sim D_X}}[g(x') - \eta_D(x')\mid g(x')=r] \nonumber \\
    w_1 &\vcentcolon= \mathop{\mathbb{P}_{x\sim D_X}}[f(x)=r_1] \nonumber \\
    w_2 &\vcentcolon= \mathop{\mathbb{P}_{x\sim D_X}}[f(x)=r_2] \nonumber \\
    w_3 &\vcentcolon= \mathop{\mathbb{P}_{x\sim D_X}}[f(x)=r] \nonumber \\
    w' &\vcentcolon=\mathop{\mathbb{P}_{x\sim D_X}}[g(x)=r] ~=~ w_1 + w_2 +w_3\nonumber
\end{align}
The latter equality holds since $g(x)=r$ if and only if $f(x) \in \{r_1, r_2, r\}$.
Now:
\begin{align} \label{eq:cepfp_cepgp}
    \ce{p, D}(f)^p - \ce{p, D}(g)^p &= |e_1|^p . w_1 + |e_2|^p . w_2 + |e_3|^p . w_3 - |e'|^p . (w_1 + w_2 + w_3)
\end{align}
Now we use the law of total expectation to rewrite $e'$ as $e_1$, $e_2$, and $e_3$:
\begin{align}
    e' &= \big[\mathop{\mathbb{E}_{x'\sim D_X}}[g(x') - \eta_D(x')\mid g(x')=r, f(x')=r_1] . w_1  \nonumber \\
    &~~~~~ + \mathop{\mathbb{E}_{x'\sim D_X}}[g(x') - \eta_D(x')\mid g(x')=r, f(x')=r_2] . w_2  \nonumber \\
    &~~~~~ +\mathop{\mathbb{E}_{x'\sim D_X}}[g(x') - \eta_D(x')\mid g(x')=r, f(x')=r] . w_3 \big]/ (w_1 + w_2 + w_3)  \nonumber \\
    &= \big[
    (r - \mathop{\mathbb{E}_{x'\sim D_X}}[\eta_D(x')\mid f(x')=r_1]) . w_1  \nonumber \\
    &~~~~~ + (r - \mathop{\mathbb{E}_{x'\sim D_X}}[\eta_D(x')\mid f(x')=r_2]) . w_2 
    \nonumber \\
    &~~~~~ + (r - \mathop{\mathbb{E}_{x'\sim D_X}}[\eta_D(x')\mid f(x')=r]) . w_3 \big]/ (w_1 + w_2 + w_3)  \nonumber \\
    &= \big[
    (r + e_1 - r_1) . w_1  + (r + e_2 - r_2) . w_2 + (e_3) . w_3 \big] / (w_1 + w_2 + w_3) 
    \nonumber \\
    &=  \big[
    e_1 . w_1  + e_2 . w_2 + e_3 . w_3 \big] / (w_1 + w_2 + w_3) \nonumber
\end{align}
\begin{align}
    \Rightarrow |e'|^p &= \big|\frac{e_1 . w_1  + e_2 . w_2 + e_3 . w_3}{w_1 + w_2 + w_3}\big|^p  \nonumber \\
    &\leq \big[\frac{|e_1| . w_1  + |e_2| . w_2 + |e_3| . w_3}{w_1 + w_2 + w_3}\big]^p
    = \big[|e_1| . w'_1  + |e_2| . w'_2 + |e_3| . w'_3\big]^p,  \nonumber
\end{align}
in which $w'_i = w_i / (w_1 + w_2 + w_3)$. So, $w'_1 + w'_2 + w'_3 = 1$.
Function $|~.~|^p$ is a convex function for any $p \in \mathop{\mathbb{N}}$. Therefore, according to Jensen's inequality \cite{jensensInequality}:
\begin{align}
    &&\big[|e_1| . w'_1  + |e_2| . w'_2 + |e_3| . w'_3\big]^p &\leq |e_1|^p . w'_1  + |e_2|^p . w'_2 + |e_3|^p . w'_3 \nonumber \\
    \Rightarrow & &|e'|^p &\leq \frac{|e_1|^p . w_1  + |e_2|^p . w_2 + |e_3|^p . w_3}{w_1 + w_2 + w_3} \nonumber \\
    \Rightarrow & &\ce{p, D}(f)^p &\geq \ce{p, D}(g)^p \quad\text{(using Equation \ref{eq:cepfp_cepgp})}\nonumber \\
    \Rightarrow & &\ce{p, D}(f) &\geq \ce{p, D}(g). \nonumber
\end{align}
\qed\end{proof}

The following two observations demonstrate the impact of merging cells on classification loss, mean squared error and the probabilistic Kendall's Tau.

\begin{observation} \label{obs:join_on_closs}
If $g$ is obtained by $(r_1,r_2)$-cell merge of $f$ with score averaging (under the conditions of Definition \ref{def:cell_merge}), then $\bLo{D}(g)$ may be smaller, larger or equal to $\bLo{D}(f)$, and the same holds for $\mathrm{MSE}_D(g)$ and $\mathrm{KT}_{D}(g)$ in comparison with $\mathrm{MSE}_D(f)$ and $\mathrm{KT}_{D}(f)$ respectively.
\end{observation}
\begin{proof}
{\sl Binary Loss:} Let $f$ be a predictor and let $a$ and $b$ be two cells generated by $f$ with $f(x) = 0.4$ for all $x\in a$ and $f(x) = 0.8$ for all $x\in b$. Let $D$ be a distribution whose marginal assigns the same probability to these cells $D_X(a) = D_X(b) = 0.1$. Further, let's assume that there are two additional, heavier cells $c$ and $d$, with $D_X(c) = D_X(d) = 0.4$, and $f(x) = 0.45$ while $\eta_D(x) = 0$ for all $x\in c$ and $f(x) = 0.55$ while $\eta_D(x) = 1$ for all $x\in d$. Thus, independently of the regression function's values in the lighter cells $a$ and $b$, the best classification threshold for $f$ will be any $\theta\in(0.45,0.55)$, say $\theta = 0.5$.

Let $g$ be obtained from $f$  by a $(0.4, 0.8)$-cell merge. Then $g(x) = 0.6$ for all $x\in a\cup b$. Since $g(x) = f(x)$ for all  $x\in c\cup d$ (the heavier cells), $g$ will also be optimally thresholded with any $\theta\in(0.45,0.55)$, thus with optimal threshold, say $\theta = 0.5$, for both $g$ and $f$ we get $g_\theta(x) = 1 \neq 0 = f_\theta(x)$ for all $x\in a$. With slight abuse of notation, we let $\eta(a) = \Pr_{(x,y)\sim D}[ y =1 ~\mid~ x\in a]$ denote the probability of label $1$ generated conditioned on cell $a$. If $\eta(a) < 0.5$, then $\bLo{D}(g) > \bLo{D}(f)$, if $\eta(a) > 0.5$, then $\bLo{D}(g) < \bLo{D}(f)$ and if $\eta(a) = 0.5$, then $\bLo{D}(g) = \bLo{D}(f)$.

{\sl Kendall's Tau Coefficient:} Let's consider the same scenario as above, but now with the four cells $a,b,c$ and $d$ having equal probability weight, say $D_X(a) =  D_X(b) = D_X(c) = D_X(d) = 0.25$. As above, we denote the conditional label probabilities in these cells by $\eta_D(a)$, $\eta_D(b)$, $\eta_D(c)$, and $\eta_D(d)$ respectively. If $\eta_D(a) < \eta_D(c) < \eta_D(d) < \eta_D(b)$,
then the scores assigned by $f$ are monotonic with respect to $\eta_D$, while the scores of $g$ are not. We thus get $1 = \mathrm{KT}_{D}(f) > \mathrm{KT}_{D}(g)$.
In case $\eta_D(c) < \eta_D(b) < \eta_D(a) < \eta_D(d)$, 
the scores of $g$ are monotonic, but the scores of $f$ are not. Thus $1 = \mathrm{KT}_{D}(g) > \mathrm{KT}_{D}(f)$.
Finally, if $\eta_D$ is a constant function, then the cell merge does not change the Kendall's Tau.

{\sl Mean Squared Error:} To show that the same phenomena can occur for the MSE, let's consider a scenario where the predictor $f$ assigns value $0$ to all points in a cell $a$ and value $1$ to all points in a cell $b$, and let's assume  $D_X(a) = D_X(b)$. Upon merging them, their combined score becomes $0.5$. When $\eta(a) = 0$ and $\eta(b) = 1$, the MSE, conditioned on these cells, increases from $0$ to $0.25$. Conversely, $\eta(a) = 1$ and $\eta(b) = 0$, the mean squared error decreases from $1$ to $0.25$. If $\eta(a) = 0.25$ and $\eta(b) = 0.75$, then the mean squared error remains the same at $0.25$.
\qed\end{proof}

\subsection{Analysis of Average Label Assignment}\label{ss:label_avg}
We now analyze another operation, where the score for every cell of predictor $f$ is replaced with the true label average in that cell. We say $\bar{f}_D: X \rightarrow [0, 1]$ with
\[
\bar{f}_D(x) \coloneqq \mathop{\mathbb{E}_{(x', y')\sim D}}[y' \mid f(x)=f(x')].
\]
\emph{is obtained by average label assignment with respect to distribution $D$} from $f$.
This true average label with respect to the data-generating distribution is typically not available to user, however might be (approximately) estimated from samples.

\begin{theorem} \label{thm:average_on_ce_mse}
For distribution $D$ over $X \times Y$ and any predictor $f: X \rightarrow \mathop{\mathbb{R}}$ 
the predictor $\bar{f}_D(x)$ obtained by average label assignment with respect to distribution $D$ from $f$ satisfies the following for $\theta = 0.5$:
\begin{align*}
    \ce{p, D}(\bar{f}_D) = 0, \qquad &
    \mathrm{MSE}_{D}(\bar{f}_D) \leq \mathrm{MSE}_{D}(f),\\
    \bLo{D}((\bar{f}_D)_{\theta}) \leq \bLo{D}(f_{\theta}),  \qquad &
    \pc{D}(\bar{f}_D) \leq \pc{D}(f).
\end{align*}
\end{theorem}
\begin{proof}
First note that $\ce{p, D}(\bar{f}_D) = 0$ is immediate from the definition. Now
let $\range{D}(f) = \{s_1^f, s_2^f, ..., s_n^f\}$ and let $\{b_1^f, b_2^f, ..., b_n^f\}$ be the corresponding cells generated by $f$.
So, for any predictor $f: X \rightarrow \mathop{\mathbb{R}}$,
\begin{align*}
    \mathrm{MSE}_{D}(f) &= \mathop{\mathbb{E}_{(x, y)\sim D}}[(y - f(x))^2] \\
    &= \sum_{i \in [1, n]} \mathop{\mathbb{E}_{(x, y)\sim D}}[(y - f(x))^2 \mid x \in b_i^f] \cdot \mathop{\mathbb{P}_{x \sim D_X}}[x \in b_i^f]  \\
    &= \sum_{i \in [1, n]} \mathop{\mathbb{E}_{(x, y)\sim D}}[(y - s_i^f)^2 \mid x \in b_i^f] \cdot \mathop{\mathbb{P}_{x \sim D_X}}[x \in b_i^f]  \\
    &= \sum_{i \in [1, n]} (
        \mathop{\mathbb{E}_{(x, y)\sim D}}[y^2 \mid x \in b_i^f] -2s_i^f \mathop{\mathbb{E}_{(x, y)\sim D}}[y \mid x \in b_i^f] 
        + (s_i^f)^2) \cdot \mathop{\mathbb{P}_{x \sim D_X}}[x \in b_i^f]
\end{align*}
Using $\bar{y}_i^f$ as $\mathop{\mathbb{E}_{(x, y)\sim D}}[y \mid x \in b_i^f]$ for any $i \in [1, n]$ and the identity $\mathop{\mathbb{E}}[X^2] = \mathop{\mathbb{V}ar}(X) + (\mathop{\mathbb{E}}[X])^2$, we rewrite the expression as:

\begin{align*}
    \mathrm{MSE}_{D}(f) &= \sum_{i \in [1, n]} (\mathop{\mathbb{V}ar_{(x, y)\sim D}}[y \mid x \in b_i^f] + (\bar{y}_i^f)^2 -2s_i^f\bar{y}_i^f + (s_i^f)^2) \cdot \mathop{\mathbb{P}_{x \sim D_X}}[x \in b_i^f] \\
    &= \sum_{i \in [1, n]} (\mathop{\mathbb{V}ar_{(x, y)\sim D}}[y \mid x \in b_i^f] + (\bar{y}_i^f-s_i^f)^2) \cdot \mathop{\mathbb{P}_{x \sim D_X}}[x \in b_i^f]
\end{align*}
The values of $f$ and $\bar{f}_D$ are different, while the cells of $f$ are a refinement of those of $\bar{f}_D$,
i.e., any two elements from $\supp{D}$ from the same cell of $f$ are also in the same cell of $\bar{f}_D$. So, $\forall i \in [1, n], b_i^{\bar{f}_D}=b_i^f, \bar{y}_i^{\bar{f}_D}=\bar{y}_i^f,$ and $\mathop{\mathbb{V}ar_{(x, y)\sim D}}[y \mid x \in b_i^{\bar{f}_D}]=\mathop{\mathbb{V}ar_{(x, y)\sim D}}[y \mid x \in b_i^f]$. Also, according to the definition of ${\bar{f}_D}$, $\forall i \in [1, n], s_i^{\bar{f}_D}=\bar{y}_i^{\bar{f}_D}$. So we complete the proof for the MSE by:
\begin{align*}
    \mathrm{MSE}_{D}(f) - \mathrm{MSE}_{D}({\bar{f}_D}) &= \sum_{i \in [1, n]} ((\bar{y}_i^f-s_i^f)^2 - (\bar{y}_i^{\bar{f}_D}-s_i^{\bar{f}_D})^2) \cdot \mathop{\mathbb{P}_{x \sim D_X}}[x \in b_i^f] \\
    &= \sum_{i \in [1, n]} (\bar{y}_i^f-s_i^f)^2 \cdot \mathop{\mathbb{P}_{x \sim D_X}}[x \in b_i^f] \geq 0
\end{align*}
Now we rewrite the classification loss of predictor $f$ using the $n$ cells:
\begin{align*}
    \bLo{D}(f_\theta) &= \Ex_{(x,y)\sim D}\indct{f_\theta(x) \neq y} \\
    &= \sum_{i \in [1, n]} \Ex_{(x,y)\sim D}[\indct{f_\theta(x) \neq y} \mid x \in b_i^f] \cdot \mathop{\mathbb{P}_{x \sim D_X}}[x \in b_i^f] \\
    &= \sum_{i \in [1, n]} \Ex_{(x,y)\sim D}[\indct{\indct{f(x) \geq \theta} \neq y} \mid x \in b_i^f] \cdot \mathop{\mathbb{P}_{x \sim D_X}}[x \in b_i^f].
\end{align*}
Let's consider the expectation part of this expression for one arbitrary cell:
\begin{align}
    & \Ex_{(x,y)\sim D}[\indct{\indct{f(x) \geq \theta} \neq y} \mid x \in b_i^f] \nonumber \\
    ~=~ & \Ex_{x\sim D_X}\Big[
    \indct{f(x) < \theta} \cdot \eta_D(x)
    ~+~  \indct{f(x) \geq \theta} \cdot (1-\eta_D(x))
    \mid x \in b_i^f\Big] \label{eq:loss_subexp_bin}
\end{align}
The classification on cell changes if and only if the label assigned to that cell (with $\theta=0.5$) changes. The label on the whole cell is constant. Without loss of generality suppose cell $b_i^f$ is labelled $0$ under $f$ and labelled $1$ under $\bar{f}_D$. Consequently, $s_i^f < 0.5$ and $\mathop{\mathbb{E}_{(x, y)\sim D}}[y \mid x \in b_i^f] \geq 0.5$, which means $\mathop{\mathbb{E}_{x\sim D_X}}[\eta_D(x) \mid x \in b_i^f] \geq 0.5$. Now we rewrite equation \ref{eq:loss_subexp_bin} for both predictors $f$ and $\bar{f}_D$:
\begin{align*}
    \Ex_{(x,y)\sim D}[\indct{\indct{f(x) \geq \theta} \neq y} \mid x \in b_i^f] & ~=~ \Ex_{x\sim D_X}[\eta_D(x)  \mid x \in b_i^f]\\
    \Ex_{(x,y)\sim D}[\indct{\indct{\bar{f}_D(x) \geq \theta} \neq y} \mid x \in b_i^f] & ~=~ \Ex_{x\sim D_X}[1 - \eta_D(x)  \mid x \in b_i^f]
\end{align*}
Since $\mathop{\mathbb{E}_{x\sim D_X}}[\eta_D(x) \mid x \in b_i^f] \geq 0.5$:
\begin{align*}
    &~~\Ex_{x\sim D_X}[1 - \eta_D(x)  \mid x \in b_i^f] \leq \Ex_{x\sim D_X}[\eta_D(x)  \mid x \in b_i^f] ~\implies~\\
    & \Ex_{(x,y)\sim D}[\indct{\indct{f(x) \geq \theta} \neq y} \mid x \in b_i^f] \geq \Ex_{(x,y)\sim D}[\indct{\indct{\bar{f}_D(x) \geq \theta} \neq y} \mid x \in b_i^f].
\end{align*}
Thus the classification loss on any arbitrary cell for predictor $f$ is greater than or equal to the loss for $\bar{f}_D(x)$, completing our proof that $\bLo{D}((\bar{f})_\theta) \leq \bLo{D}(f_\theta)$.

When the initial scores are replaced with the average of true labels, it is possible for two cells of $f$ to have equal new scores. In such cases, these cells are merged, and by Theorem \ref{thm:join_on_pc}, it follows that $\pc{D}(\bar{f}_D) \leq \pc{D}(f)$.
\qed\end{proof}
Next we present an analogous result for monotonicity.
\begin{observation} \label{obs:average_on_kt}
Let $\bar{f}_D(x)$ be the predictor obtained by average label assignment with respect to some distribution $D$ from some predictor $f$. Then the probabilistic Kendall's Tau coefficient of  $\bar{f}_D(x)$ may be smaller, larger or equal to the Kendall's Tau coefficient of $f$ with respect to $D$.
\end{observation}
    \begin{proof}
Let's consider a predictor $f$ that generates only two distinct values $r_a$ and $r_b$, say $r_a < r_b$ and let $a$ and $b$ be the corresponding cells. 
As before, we denote denote the expected label in these cells by $\eta_D(a)$ and $\eta_D(b)$.
        
We first consider the case that both cells $a$ and $b$ contains four distinct domain points, and the values of $\eta_D$ on the four points in cell $a$ are three times $0.1$, and once $1$, and $0.2$ for all four points in $b$. Thus $\eta_D(a) = 0.325$ and $\eta_D(b) = 0.2$, and 
        \begin{align*}
            &\mathrm{KT}_{D}(f) = 1 - 2 \cdot \frac{8}{56} = \frac{5}{7},
            &\mathrm{KT}_{D}(\bar{f}_D) = 1 - 2 \cdot \frac{24}{56} = \frac{1}{7}.
        \end{align*}
 Thus, in this case, substituting the scores with the average true labels has weakened the monotonicity of the predictor.

 Now consider the same scenario with the only difference being that the values of $\eta_D$ for the points in region $a$ are now $0.1$, twice $0.3$, and $1$. In this case the monotonicity, as measured by the probabilistic Kendall's Tau, has improved:
        \begin{align*}
            &\mathrm{KT}_{D}(f) = 1 - 2 \cdot \frac{24}{56} = \frac{1}{7},
            &\mathrm{KT}_{D}(\bar{f}_D) = 1 - 2 \cdot \frac{8}{56} = \frac{5}{7}.
        \end{align*}        
 Lastly, if the regression function is a constant function, then $\mathrm{KT}_{D}(f) =\mathrm{KT}_{D}(\bar{f}_D)$.
\qed\end{proof}

\section{Experimental Evaluation of Decision Tree Based Models}\label{sec:experiments}
In our experimental evaluation we compare standard methods for calibration to a simple model that is inherently interpretable, namely a Decision Tree (DT). 
Note that for a decision tree, the induced cells are inherently interpretable, and it is also straght-forward to control the number of cells, thus this methods straightforwardly satisfies our basic requirements for interpretability.
Our goal is then to determine how this simple interpretable method compares to non-interpretable standard methods, in terms of our remaining desiderata, namely calibration, classification accuracy, approximation of the regression function and monotonicity. 
We include two standard methods for calibration through post-processing, namely  Platt Scaling (PS) \cite{plattscaling} and Isotonic Regression \cite{isotonic} in our comparison. A support vector machine (SVM) is first trained as the base model for Platt Scaling and Isotonic Regression. Additionally we compare to another tree based calibration model, Probability Calibration Tree (PCT) \cite{PCT}. 

\subsubsection*{Evaluation Metrics Employed}    
Since some of our desiderata, namely calibration, approximation of the regression function and monotonicity directly depend on the unknown values of the regression function $\eta_D$, there is no immediate way to assess these from finite data. 
We employ commonly used metrics that we list below, as well as a novel metric that we introduce.
For classification accuracy we evaluate the empirical binary loss \cite{closs}; 
for calibration we evaluate the (empirical) Expected calibration error (ECE) \cite{metric_ece}; for approximation of the regression function we evaluate the Root Mean Square Error (RMSE)\cite{RMSE}, and for monotonicity we evaluate the Area under the ROC curve (AUC)\cite{AUC}. Another metric for calibration that we evaluate is the Area under the Validity Curve (AUC$_V$)\cite{metric_validity_auc}. And in addition to these metrics from the literature, we introduce a novel calibration metric that we term Probability Deviation Error (PDE). 

For the definitions of these metrics below, we let $x_i\in X$ and $y_i \in \{0, 1\}$ denote the features and label of a single sample. $D^{(n)}$ denotes the collection of $n$ samples:
\[
D^{(n)} \vcentcolon= ((x_1, y_1), (x_2, y_2), ..., (x_n, y_n))
\]
\paragraph{Classification (0/1)-Loss} In our experiments, we utilize the threshold $\theta = 0.5$ for
\[
\bLo{n}(f_\theta) = \frac{1}{n}\sum_{i=1}^{n}
\begin{cases}
(y_i)^2 & \text{if $f(x_i) \leq \theta$} \\
(1-y_i)^2 & \text{otherwise}.
\end{cases}
\]

\paragraph{Root Mean Square Error (RMSE) \cite{RMSE}}
\[\rmse{n}(f) = \left(\frac{1}{n}\sum_{i=1}^{n}(f(x_i) - y_i)^2\right)^{1/2}\]

\paragraph{Area Under the ROC Curve (AUC)\cite{AUC}} Given a predictor $f$, using different thresholds $\theta\in[0,1]$, we obtain classifiers $f_\theta$ with increasing
True Positive and False Positive Rates (TPR and FPR) over the sample points. 
These pairs of rates yield curve (where TPR is viewed as a function of FPR), and $\auc$ is defined as the area under this curve.
This is standard metric to evaluate the monotonicity of a predictor with respect to the regression function. 
If the model generates wrong scores (in terms of pointwise probability estimates), but in the correct order (that $f$ is monotonic with respect to $\eta_D$), then $\auc$ is still high. 

\paragraph{Expected Calibration Error (ECE)}
This criterion compares the average predicted scores and the average of true labels with respect to a given set of bins $b_1, b_2, ..., b_B$, where the bins form a partition of the space or dataset \cite{metric_ece,MBCT,FamigliniCC23}. We let $w_i$ denote the fraction of data points contained in bin $b_i$. 
ECE is then defined as follows:
\[
\mathrm{ECE}_{p} \vcentcolon= \big(\sum_{i=1}^{B} w_i \cdot \mathrm{PCE}(b_i)^p \big)^{1/p}.
\]
where $\mathrm{PCE}$ is the \emph{Partition Calibration Error}, which is the difference between the average values generated by $f$ and the average of labels in a bin:
\[
\mathrm{PCE}(b_i) \vcentcolon= \frac{| \sum_{j=1}^{n} (f(x_j)-y_j)  \indct{x_j \in b_i}|}{\sum_{j=1}^{n}\indct{x_j \in b_i}}
= \left|\frac{1}{|b_i|}\sum_{x_j\in b_i}f(x_j) - \frac{1}{|b_i|}\sum_{x_j\in b_i} y_j \right|
\]
When no binning is provided, uniform mass binning is employed, that is the produced scores are sorted and allocated into a fixed number $B$ of equally weighted bins. The resulting criterion is denoted by $
\mathrm{ECE}_{B, p} \vcentcolon= \big(\frac{1}{B} \sum_{i=1}^{B} \mathrm{PCE}(b_i)^p \big)^{1/p}.
$
While $\mathrm{ECE}$ is widely used to evaluate calibration models \cite{naeini2015bayesian,guo2017calibration,blasiok2022unifying,FamigliniCC23}, it can be a problematic measure when the bins used do not correspond to the actual cells of the predictor $f$.  ECE then effectively evaluates a different predictor, namely the predictor that results from $f$ when the scores are averaged in each given bin. We discuss how this can result in distorted conclusion in Appendix Section \ref{sec:critique_ece}.

\paragraph{Probability Deviation Error (PDE)}
To address the issues of ECE (see appendix Section \ref{sec:critique_ece}), we propose a new metric which we term Probability Deviation Error (PDE). The PDE compares the point-wise scores with the average label in each bin, thereby fixing the problem associated with the ECE. On predefined bins $b_1, b_2, ..., b_B$ with weights $w_1, w_2, ..., w_B$, the $L_p$ norm $\mathrm{PDE}$ is defined as follows:
\[
\mathrm{PDE}_p \vcentcolon= \big(\sum_{i=1}^{B} w_i \cdot \mathrm{PPD}(b_i)^p \big)^{1/p}.
\]
where $\mathrm{PPD}$, or \emph{Partition Probability Deviation} is the average difference between point-wise scores generated by $f$ in a bin and the label average in the bin:
\[
\mathrm{PPD}(b_i) \vcentcolon= \frac{\sum_{j=1}^{n} |f(x_j)-\hat{y}_i|  \indct{x_j \in b_i}}{\sum_{j=1}^{n}\indct{x_j \in b_i}}
=
\frac{1}{|b_i|}\sum_{x_j\in b_i} |f(x_j)-\hat{y}_i|
\]
where $\hat{y}_i = \frac{1}{|b_i|}\sum_{x_k\in b_i} y_k$ is the average label in bin $b_i$.
 If no partition into bins is given, as for ECE, uniform mass binning with $B$ bins is used by default. In this case, we denote the criterion as $\mathrm{PDE}_{B,p}\vcentcolon= \big(\frac{1}{B} \sum_{i=1}^{B} \mathrm{PPD}(b_i)^p \big)^{1/p}$. We illustrate that this metric better reflects quality of calibration than ECE, by empirically comparing these two measures  on synthetically generated data (see Appendix section \ref{sec:eval_pde}).

\paragraph{Area Under the Validity Curve ($\auc_V$)}
This metric has recently been proposed to evaluate calibration \cite{metric_validity_auc}.
We first define the \emph{validity function $V:\reals^X \times [0,1]\to[0,1]$} that assigns to each threshold $\epsilon\in[0,1]$ the probability mass of the area where predictor $f$ is $\epsilon$-valid as measured by the $L_1$ norm calibration error: 
$$V(f, \epsilon) = \Pr_{(x, y)\sim D}[|f(x) - \mathop{\mathbb{E}_{(x', y')\sim D}}[y'~|~f(x')=f(x)]| \leq \epsilon]$$
This function generates a curve, the \emph{validity curve}, whose integral over $[0,1]$ is the metric $\auc_V(f)$ \cite{metric_validity_auc}. Its relation to the \emph{$L_1$ norm calibration error} for $f:\X\to[0,1]$ has been shown to satisfy $\mathrm{CE}_1(f) ~= ~1 - \mathrm{AUC}_V(f)$ \cite{metric_validity_auc}.

With finite data $D^{(n)}$, the validity function $V$ is estimated as follows \cite{metric_validity_auc}:
\[
\hat{V}(f, \epsilon)=\frac{1}{n} \sum_{i=1}^{n} \indct{|f(x_i) - \mathop{\hat{\mathbb{E}}_{(x, y)\sim D^{(n)}}}[y|f(x) = f(x_i)]| \leq \epsilon}
\]
where $\mathop{\hat{\mathbb{E}}_{(x, y)\sim D^{(n)}}}[y~|~f(x) = p]| \vcentcolon= \frac{\sum_{i=1}^{n} y_i  \indct{f(x_i) = p}}{\sum_{i=1}^{n} \indct{f(x_i) = p}}$. 
This latter empirical expectation counts the number of samples with the same score. But on a finite dataset no two (or very few) points may have the same score. The measure thus requires an appropriate binning method. Prior work \cite{metric_validity_auc} involved averaging scores over the uniform mass bins as in ECE, effectively evaluating a different predictor.

We used a different method to estimate the validity function with finite number of samples, based on the K-nearest neighbors (KNN). 
K-nearest neighbor based $\mathrm{AUC}_V$ estimation
is the area under the following estimated validity curve: 
\[
\hat{V}_\mathrm{KNN}(f, \epsilon)=\frac{1}{n} \sum_{i=1}^{n} \indct{|f(x_i) - \mathop{\hat{\mathbb{E}}_{\mathrm{KNN} (x, y)\sim D^{(n)}}}[y|f(x) = f(x_i)]| \leq \epsilon},
\]
where the empirical expectation estimation is
\[
\mathop{\hat{\mathbb{E}}_{\mathrm{KNN} (x, y)\sim D^{(n)}}}[y~|~f(x) = p]| \vcentcolon= \frac{\sum_{i=1}^{n} y_i  \indct{f(x_i) \in \operatorname{k-nn}(p)}}{k}
\]
with $\operatorname{k-nn}(p)$ being the set of $k$ samples with the closest $f$-scores to $p$.

Our method takes into account the scores that predictor $f$ assigns to each datapoint, instead of relying on the average scores over different bins. We thus avoid evaluating a modified version of the model. 
We denote the KNN based $\mathrm{AUC}_V$ estimation by $\mathrm{AUC}_{V, \mathrm{KNN}}$ and employed this in our experiments.

\subsubsection*{Datasets} \label{sec:real_datasets}
We used 36 datasets for binary and multi-class classification tasks. All 36 datasets are from UCI \cite{UCI} and their properties are summarized in Table \ref{tab:datasets}.
\begin{table}
\caption[Real world datasets]{\label{tab:datasets}Real world datasets used in our experiments.}
\begin{scriptsize}
\begin{minipage}{0.49\textwidth}
\centering
\begin{tabular}{lccc}
\hline
Dataset                & Instances & Attributes & Classes \\ \hline
audiology              & 226       & 69         & 24      \\
bank-marketing         & 41188     & 19         & 2       \\
bankruptcy             & 10503     & 64         & 2       \\
car-evaluation         & 1728      & 6          & 4       \\
cervical-cancer        & 858       & 32         & 2       \\
colposcopy             & 287       & 62         & 2       \\
credit-rating          & 690       & 15         & 2       \\
cylinder-bands         & 512       & 39         & 2       \\
german-credit          & 1000      & 20         & 2       \\
hand-postures          & 78095     & 39         & 2       \\
htru2                  & 17898     & 8          & 2       \\
iris                   & 150       & 4          & 3       \\
kr-vs-kp               & 3196      & 36         & 2       \\
mfeat-factors          & 2000      & 216        & 10      \\
mfeat-fourier          & 2000      & 76         & 10      \\
mfeat-karhunen         & 2000      & 64         & 10      \\
mfeat-morph            & 2000      & 6          & 10      \\
mfeat-pixel            & 2000      & 240        & 10    
\end{tabular}
\end{minipage}
\hfill
\begin{minipage}{0.49\textwidth}
\centering
\begin{tabular}{lccc}
\hline
Dataset                & Instances & Attributes & Classes \\ \hline
mice-protein           & 1080      & 80         & 8       \\
new-thyroid            & 215       & 5          & 3       \\
news-popularity        & 39644     & 59         & 2       \\
nursery                & 12960     & 8          & 5       \\
optdigits              & 5620      & 64         & 10      \\
page-blocks            & 5473      & 10         & 5       \\
pendigits              & 10992     & 16         & 10      \\
phishing               & 1353      & 10         & 3       \\
pima-diabetes          & 768       & 8          & 2       \\
segment                & 2310      & 20         & 7       \\
shuttle                & 58000     & 9          & 7       \\
sick                   & 3772      & 29         & 2       \\
spambase               & 4601      & 57         & 2       \\
taiwan-credit          & 30000     & 23         & 2       \\
tic-tac-toe            & 958       & 9          & 2       \\
vote                   & 435       & 16         & 2       \\
vowel                  & 990       & 14         & 10      \\
yeast                  & 1484      & 8          & 10
  
\end{tabular}
\end{minipage}
\end{scriptsize}
\end{table}

\subsubsection*{Results}
We evaluate the listed metrics of the four calibration methods on all real $36$ world datasets. For each dataset and method, we average over $10$ repetitions, and each time randomly partitioning the samples into training, calibration, and test sets ($40.5\%$, $49.5\%$ and $10\%$ respectively).
An SVM with Gaussian kernel is the base model for PS and IR post-processing calibration methods. 
Training of the tree-based calibration models (PCT and DT) involved a cost-complexity post-pruning step. The optimal cost-complexity parameter is found via $5$-fold cross-validation on the calibration set. Afterwards, the whole calibration set is used to train the calibration model and the model is pruned using the found optimal factor. 
To evaluate the models with L1-norm-ECE, we use uniform-mass binning with $32$ bins. For L1-norm-PDE, we used the leaves of the tree for models PCT and DT. For DT this corresponds to the cells generated by the predictor. There are no meaningful cells for PS and IR, thus no PDE is reported. 
\begin{table*}[h]
\caption{\label{tab:win_count}The calibration methods are compared using the above metrics. We report the number of times each method is the top performer. The numbers in each column do not add up to the number of datasets as multiple models may have won simultaneously. The pre-fixed bins for metric PDE are the leaves generated by PCT and DT.
We use $k=10$ for  $\mathrm{AUC}_{V, \mathrm{KNN}}$.
}
\centering
\begin{tabular}{lcccccc}
\hline
Method  & RMSE & 0/1-loss & AUC & $\mathrm{ECE}_{B=32, p=1}$            & $\mathrm{PDE}_{p=1}$ & $\mathrm{AUC}_{V, \mathrm{KNN}}$   \\ \hline
\multicolumn{1}{l|}{PS}             & 21             & 22                             & 27            & 7                 & -  &  5  \\
\multicolumn{1}{l|}{IR}             & 25            & 24                            & 28            & 14                 & -  & 9   \\
\multicolumn{1}{l|}{PCT}            & \textbf{31}   & \textbf{31}                   & \textbf{32}   & \textbf{26}       & 19  & \textbf{15}   \\
\multicolumn{1}{l|}{\textbf{DT}}    & 24            & 24                            & 19             & 21                & \textbf{36} & 7    \\ \hline 
\end{tabular}
\end{table*}

Table \ref{tab:win_count} summarizes the results over the $36$ datasets. For each of the four methods and six metrics, we report how often the method obtained the best score. We count cases of ties towards all winning methods
(which is why the columns in that table don't always sum up to $36$). 
Note that the simple decision tree (DT) is the only predictor evaluated here that can be considered interpretable. The other three methods produce infinitely many different scores (their effective range is infinite), and a user can not reasonably be expected to have a notion of the shapes of the resulting cells. The summary shows that in our experiments the simple decision tree is a predictor that performs similarly well as PS and IR on all metrics and performs best among all four methods in terms of PDE. The PCT method outperforms DT on most metrics. However we would argue that the overall performance of DT is a worth-while trade-off for interpretability.

\section{Concluding Discussion}
The goal of this work is to provide a systematic framework for understanding different aspects of calibration and to highlight the importance of taking interpretability into account when promoting calibration. Calibration is a notion that is inherently aimed at providing users with better understanding of label certainty. Our axiomatic framing and analysis highlight which aspects of a calibrated predictor can improve human comprehension of the provided scores (namely interpretable cells, not too large number of cells and monotonicity with respect to the regression function of the data-generating process), and show how some aspects fulfill other important purposes (accuracy, and pointwise approximation of the regression function). The three levels of analysis (axioms/properties, distributional measures of distance from these and empirical measures) further clarify the higher level concepts that frequently cited empirical measures are aimed at. 

Providing confidence scores to end users without a way of clearly communicating the meaning and range of validity of these scores might pose more risks in terms of effects of downstream decisions than not providing any confidence scores at all (for example when a high confidence score instills a false sense of certainty). We hope that our work contributes to and will inspire more investigations into  interpretability for calibrated scores.

\subsubsection*{Acknowledgements}
This research was funded by an NSERC discovery grant.\\
{\bf Disclosure of Interests.} The authors have no competing interests to declare that are
relevant to the content of this article.

\bibliographystyle{unsrtnat}
\bibliography{references}  






\appendix
\section{Exploring the Probabilistic Count (PC)}\label{sec:probabilistic_count}
In this section, we exhibit characteristics of our newly introduced measure, the probabilistic count (PC). We first establish a connection between PC and the true size of the effective range. Subsequently, we provide some illustrative examples.

\begin{theorem} \label{thm:pc_upperbound}
For distribution $D$ over $X \times Y$ and predictor $f: X \rightarrow \mathop{\mathbb{R}}$,
we have 
$\pc{D}(f) \leq |\range{D}(f)|$. Equality $\pc{D}(f) = |\range{D}(f)|$ holds if and only if all cells of $f$ have the same probability.
\end{theorem}
\begin{proof}
Assume $\range{D}(f)=n = \{r_1, r_2, \ldots, r_n\}$, and $p_i \coloneqq \mathop{\mathbb{P}_{x \sim D_X}}[f(x) = r_i]$.
\begin{align*}
    \frac{1}{\pc{D}(f)} &= \mathop{\mathbb{P}_{x, x'\sim D_X}} [f(x) = f(x')] \\
    &= \sum_{i=1}^{n}\big[ \mathop{\mathbb{P}_{x, x' \sim D_X}}[f(x)=f(x') | f(x) = r_i] ~\cdot~ \mathop{\mathbb{P}_{x, x' \sim D_X}}[f(x) = r_i]\big] \\
    &=  \sum_{i=1}^{n}\big[ \mathop{\mathbb{P}_{x' \sim D_X}}[f(x')=r_i] ~\cdot~ \mathop{\mathbb{P}_{x \sim D_X}}[f(x) = r_i]\big] 
    ~=~ \sum_{i=1}^{n}p_i^2. 
\end{align*}
We define two vectors with size $n$ as $u = (p_1, ..., p_n)$ and $v = (1, ..., 1)$. Using Cauchy-Schwarz inequality $|\langle u, v \rangle|^2 \leq |\langle u, u \rangle| \cdot |\langle v, v \rangle|$
and the equality holds if and only if $u$ and $v$ are parallel. With this we get $ |\langle u, v \rangle|^2 = 
    (\sum_{i=1}^{n} p_i)^2 \leq (\sum_{i=1}^{n} p_i^2) \cdot n$ which implies $\sum_{i=1}^{n} p_i^2 \geq \frac{1}{n}$.
So, $\pc{D}(f) \leq n$. The equality holds if and only if $(p_1, ..., p_n)$ and $(1, ..., 1)$ are parallel which means that all $p_i$s are equal.
\qed\end{proof}
The probabilistic count depends on the number of cells and their probability weights. 
We here present a series of examples of this metric. In each example, the cells created in the range of $f$ from $\supp{D}$ are visualized in a bar. The length of each partition represents its probability in the distribution $D$. 

PC is not monotonic with the actual number of cells. The function in Figure \ref{fig:PC:outlier} generates 9 cells while its PC is 4.28. If we compare this predictor with a function that generates 5 balanced cells, which leads to PC equals to 5, we can show that the predictor in Figure \ref{fig:PC:outlier} has less PC while it has more cells.
    \begin{figure}[!ht]
        \begin{center}
        \includegraphics[width=0.7\textwidth]{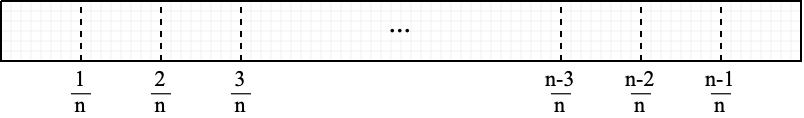}
        \caption{Probabilistic count example on $n$ cells with the same weight; in  this case $\pc{D}(f) =n$ is the number of cells. 
    }
        \label{fig:PC:balanced}
        \end{center}
    \end{figure}
    \begin{figure}[!ht]
        \begin{center}
        \includegraphics[width=0.7\textwidth]{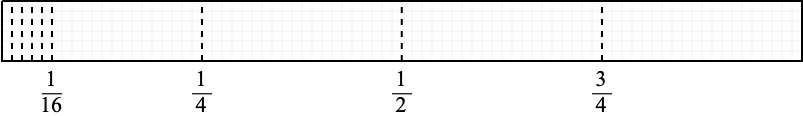}
        \caption{Probabilistic count example on 9 cells including 5 small cells; the cells with small weights do not have much effect on the probabilistic count; while we have 9 cells, $\pc{D}(f)$ is close to $4$; this shows that PC emphasizes the number of significant cells. $\pc{D}(f)=\frac{1}{\frac{1}{4}^2\cdot3 + \frac{19}{80}^2 + \frac{1}{80}^2 \cdot 5} \approx 4.08$.}
        \label{fig:PC:outlier}
        \end{center}
    \end{figure}
    \begin{figure}[!ht]
        \begin{center}
        \includegraphics[width=0.7\textwidth]{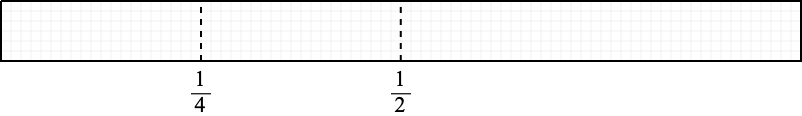}
        \caption{Probabilistic count example on three cells with different weights; $\pc{D}(f)=\frac{1}{\frac{1}{4}^2\cdot2 + \frac{1}{2}^2} \approx 2.66$.}
        \label{fig:PC:imbalanced}
        \end{center}
    \end{figure}    
    \begin{figure}[!ht]
        \begin{center}
        \includegraphics[width=0.7\textwidth]{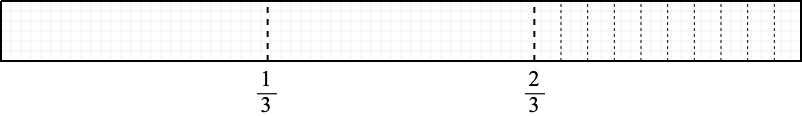}
        \caption{Probabilistic count example on 12 cells including 10 cells distributed on a third; for three equally weighted cells, the probabilistic count is 3; we have split the third partition into ten small cells with the same weights. $\pc{D}(f)=\frac{1}{\frac{1}{3}^2\cdot2 + \frac{1}{30}^2\cdot10} \approx 4.28$.}
        \label{fig:PC:3_with_1}
        \end{center}
    \end{figure}
    \begin{figure}[!ht]
        \begin{center}
        \includegraphics[width=0.7\textwidth]{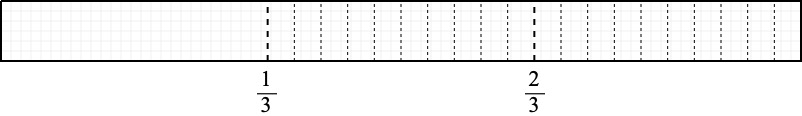}
        \caption{Probabilistic count example on 21 cells including 20 cells distributed on two thirds; we have split each of the second and the third cells into ten small cells with the same weights; $\pc{D}(f)=\frac{1}{\frac{1}{3}^2 + \frac{1}{30}^2\cdot20} = 7.5$.\\}
        \label{fig:PC:3_with_2}
        \end{center}
    \end{figure}
\section{Critiquing the Expected Calibration Error} \label{sec:critique_ece}
The empirical expected calibration error (ECE) is a metric that is frequently employed to measure calibration \cite{naeini2015bayesian,guo2017calibration,blasiok2022unifying}. It averages scores within each bin, rather than evaluating the individual scores, which we show makes it less accurate. When a bin contains both overconfident and underconfident scores, they average out, making the performance seem better than it is, as illustrated in Example \ref{exmp:ece_criticism}.

\begin{example} \label{exmp:ece_criticism}
Consider a predictor $f$ evaluated with a partition that contains a bin $b$ with $f(x_i) = 0.35$ for half the samples, and $f(x_i) = 0.65$ for the other half. Assume that on this bin the regression function satisfies $\eta_D(x) = 0.5$ for all $x$, and that there re sufficiently many sample from the bin that the empirical average is close to $0.5$.
Now when evaluating the ECE  on this bin, the result would be $|\frac{0.35 + 0.65}{2} - 0.5| = 0$, indicating flawless performance of $f$ in terms of calibration, which is not correct.
In contrast, the probability deviation error (PDE) takes individual scores into account. For the same bin, the PDE evaluates to $\frac{|0.35 - 0.5| + |0.65 - 0.5|}{2} = 0.15$, which corresponds to the correct calibration error.
\end{example}

\section{Empirically Motivating the Probability Deviation Error}
\label{sec:eval_pde}
To motivate our proposed measure PDE beyond the Example \ref{exmp:ece_criticism} above, we empirically compare PDE with ECE on a large synthetic dataset.
We compare their \emph{bias},
where the \emph{bias} of a calibration metric $\mu$ for predictor $f:\X\to[0,1]$ over $n$ samples $D^{(n)}$ with respect to the distribution $D$ is defined as (\cite{MBCT}):   
\begin{align*}
   \mathrm{Bias}(D, D^{(n)}, \mu) \coloneqq \frac{1}{m} \sum_{i=1}^{m} \mu(D^{(n)}(f) - \frac{1}{n}\sum_{i=1}^{n}(|f(x_i) - \eta_D(x_i)|),
\end{align*}
where $m$ is the number of experiments, each over datasets of size $n$. We used $m=10$.
Since the regression function is essential to assess bias, we synthetically generated $27,500$ samples, generating labels according to a predefined regression function. Thus, we have access to $\eta_D(x_i)$ for our generated points. The generated samples are subsequently employed to train a decision tree or a PCT. In experiments where a PCT is trained, an additional $22,500$ samples are generated to train an SVM, which serves as the base model required for the PCT.
We used different sizes of test sets (generated with the same procedure), and
Uniform-mass binning for  ECE and PDE with a different number of bins ranging from $2$ to $64$. Our analysis indicate that PDE has mostly lower bias than ECE, provided there are enough samples per bin, see Figure \ref{fig:ece_pde_bias_umass} below.
\begin{figure}
    \centering
    \begin{subfigure}{0.48\textwidth}
        \centering
        \includegraphics[trim={0 0 0 1.3cm},clip,width=\textwidth]{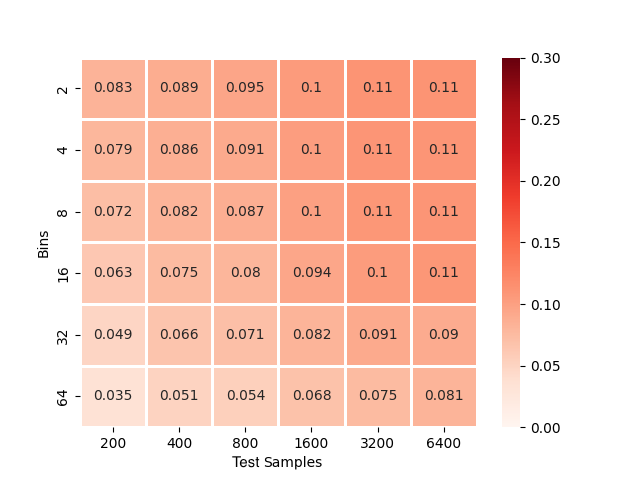}
        \caption{\small{ECE of DT}}
    \end{subfigure}
    \begin{subfigure}{0.48\textwidth}
        \centering
        \includegraphics[trim={0 0 0 1.3cm},clip,width=\textwidth]{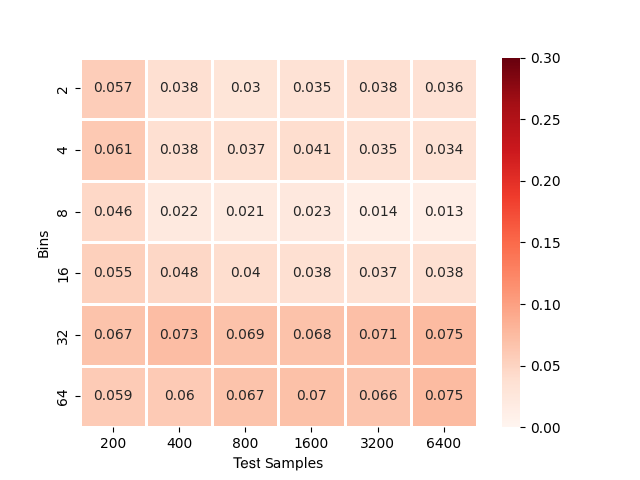}
        \caption{\small{PDE of DT}}
    \end{subfigure}

    \begin{subfigure}{0.48\textwidth}
        \centering
        \includegraphics[trim={0 0 0 1.3cm},clip,width=\textwidth]{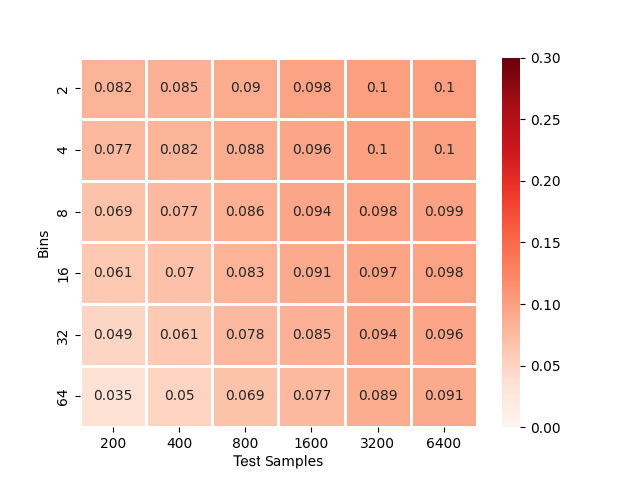}
        \caption{\small{ECE of PCT}}
    \end{subfigure}
    \begin{subfigure}{0.48\textwidth}
        \centering
        \includegraphics[trim={0 0 0 1.3cm},clip,width=\textwidth]{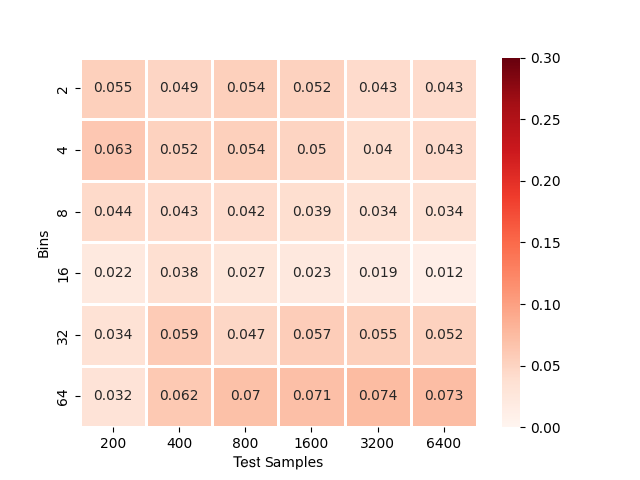}
        \caption{\small{PDE of PCT}}
    \end{subfigure}

    \caption[Calibration metric biases using uniform-mass binning]{
     Evaluating the bias of ECE and PDE using uniform-mass binning, with the former exhibiting larger bias in most cases, especially with larger numbers of test samples}. 
    \label{fig:ece_pde_bias_umass}
\end{figure}

We have proposed a Breadth First Search Leaf (BFSL) binning approach for partitioning samples when evaluating a model, as an alternative to the conventional uniform-mass binning method. This technique is applicable to any tree-based calibration models. There are several calibration models that are using a tree structure (\cite{PCT, MBCT}). Using BFSL, we take advantage of the structure of the calibration models to define the regions. The main motivation of this approach is to provide a more interpretable and explainable way of partitioning the predicted probabilities as interpretability is one of the desired properties in calibration framework.
First, by performing a breadth first search starting from the root of the original tree, the shortest sub-tree with $B$ leaves is extracted, in which $B$ is the required number of bins. The leaves of this sub-tree are the regions generated by this method. In this method, samples that follow similar paths in the tree are in the same region. The number of bins can be chosen based on the specific requirements of the analysis, and the resulting bins are expected to capture the properties of the model's behavior.
The benefits of this approach include the ability to interpret regions due to the utilization of the tree's structure. Furthermore, these regions rely on the characteristics of the samples, not just their scores. In some of the experiments, we have used the leaves generated by the original tree in the calibration model without performing breadth first search. This partitioning allows us to evaluate the model using the structure of the tree without additional pruning. We have analyzed the bias of ECE and PDE using BFSL binning method in Figure \ref{fig:ece_pde_bias_bfsl}.

\begin{figure}
    \centering
    \begin{subfigure}{0.48\textwidth}
        \centering
        \includegraphics[trim={0 0 0 1.3cm},clip,width=\textwidth]{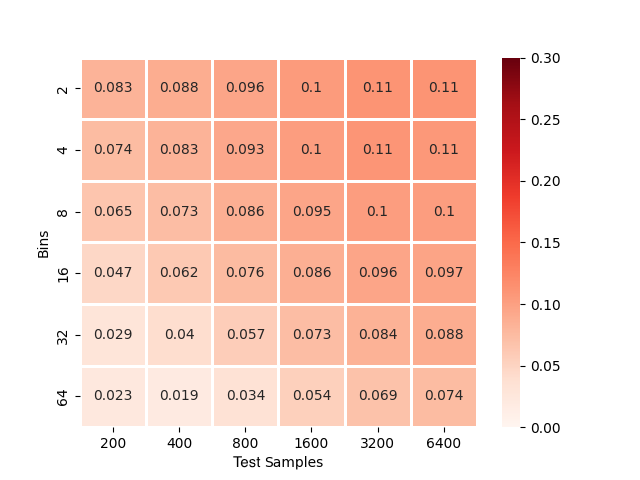}
        \caption{\small{ECE of DT}}
    \end{subfigure}
    \begin{subfigure}{0.48\textwidth}
        \centering
        \includegraphics[trim={0 0 0 1.3cm},clip,width=\textwidth]{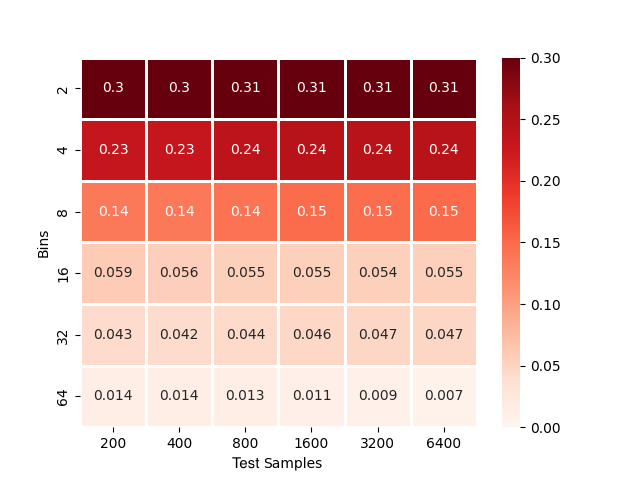}
        \caption{\small{PDE of DT}}
    \end{subfigure}

    \begin{subfigure}{0.48\textwidth}
        \centering
        \includegraphics[trim={0 0 0 1.3cm},clip,width=\textwidth]{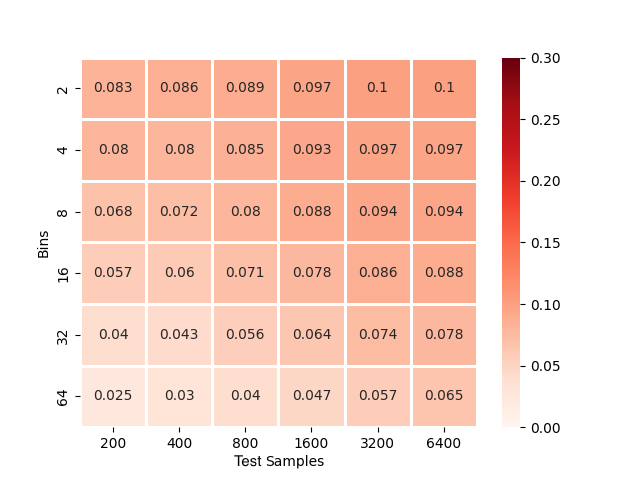}
        \caption{\small{ECE of PCT}}
    \end{subfigure}
    \begin{subfigure}{0.48\textwidth}
        \centering
        \includegraphics[trim={0 0 0 1.3cm},clip,width=\textwidth]{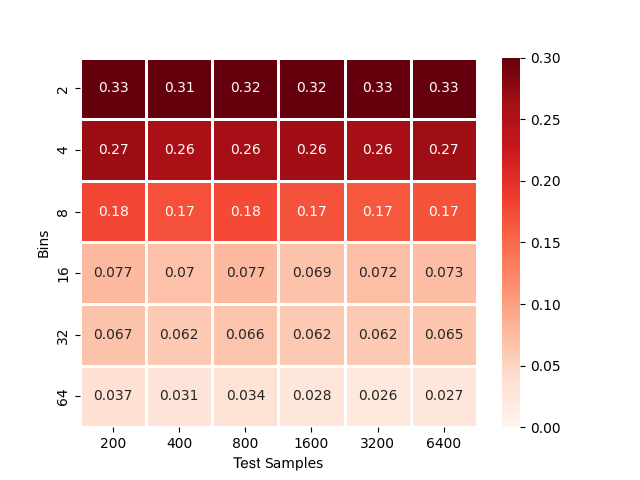}
        \caption{\small{PDE of PCT}}
    \end{subfigure}

    \caption[Calibration metric biases using BFSL binning]{
     Evaluating the bias of ECE and PDE using BFSL binning, with the former exhibiting larger bias in most cases, especially with larger numbers of test samples}. 
    \label{fig:ece_pde_bias_bfsl}
\end{figure}

To conduct a more thorough investigation of this experiment and its analysis, we employed a more sophisticated synthetic data generator and repeated the experiment under identical conditions but with different data samples. The results of this experiment are presented in Figure \ref{fig:ece_pde_bias_umass_complexdata} and Figure \ref{fig:ece_pde_bias_bfsl_complexdata}.

\begin{figure}
    \centering
    \begin{subfigure}{0.48\textwidth}
        \centering
        \includegraphics[trim={0 0 0 1.3cm},clip,width=\textwidth]{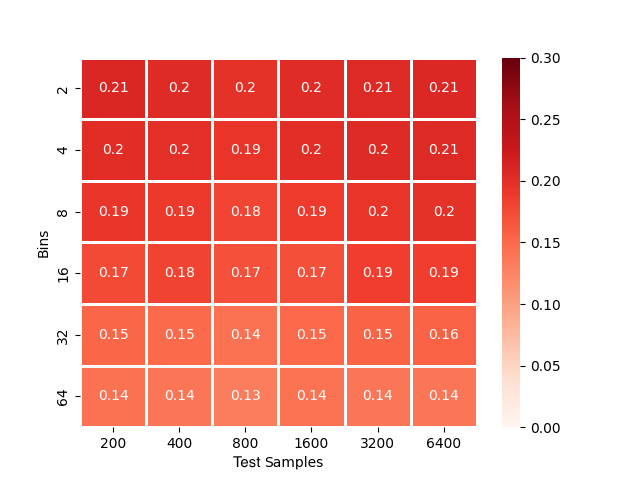}
        \caption{\small{ECE of DT}}
    \end{subfigure}
    \begin{subfigure}{0.48\textwidth}
        \centering
        \includegraphics[trim={0 0 0 1.3cm},clip,width=\textwidth]{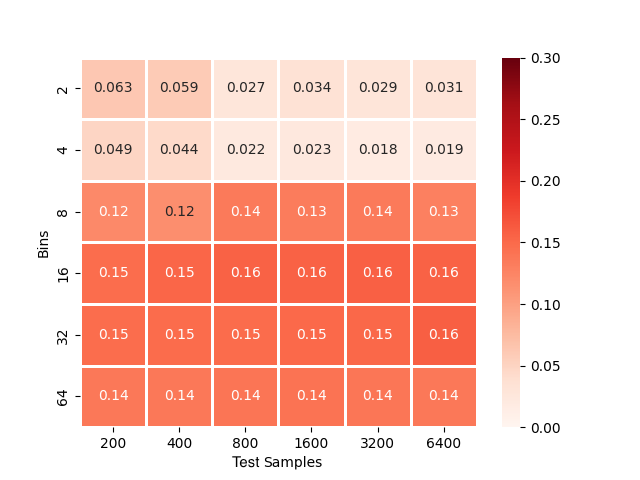}
        \caption{\small{PDE of DT}}
    \end{subfigure}

    \begin{subfigure}{0.48\textwidth}
        \centering
        \includegraphics[trim={0 0 0 1.3cm},clip,width=\textwidth]{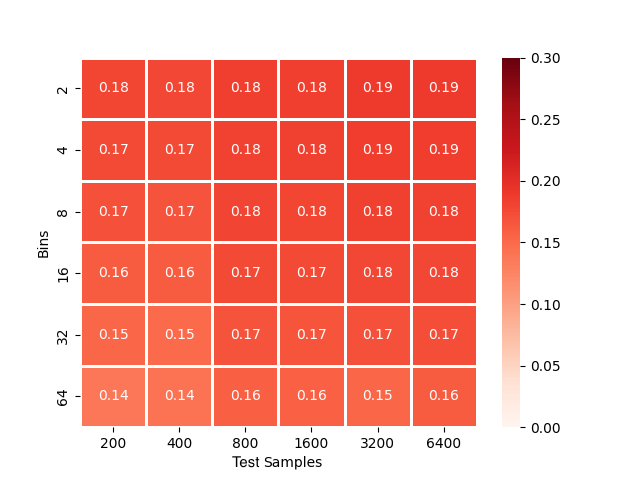}
        \caption{\small{ECE of PCT}}
    \end{subfigure}
    \begin{subfigure}{0.48\textwidth}
        \centering
        \includegraphics[trim={0 0 0 1.3cm},clip,width=\textwidth]{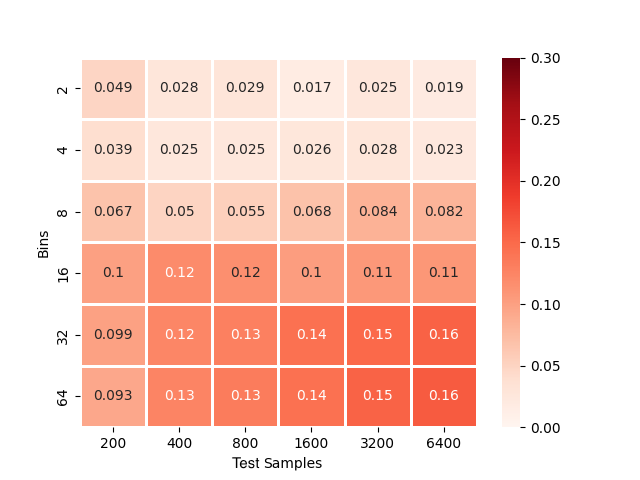}
        \caption{\small{PDE of PCT}}
    \end{subfigure}

    \caption[Calibration metric biases using uniform-mass binning on complex data]{
     Evaluating the bias of ECE and PDE using uniform-mass binning on complex data, with the former exhibiting larger bias in most cases, especially with larger numbers of test samples}. 
    \label{fig:ece_pde_bias_umass_complexdata}
\end{figure}

\begin{figure}
    \centering
    \begin{subfigure}{0.48\textwidth}
        \centering
        \includegraphics[trim={0 0 0 1.3cm},clip,width=\textwidth]{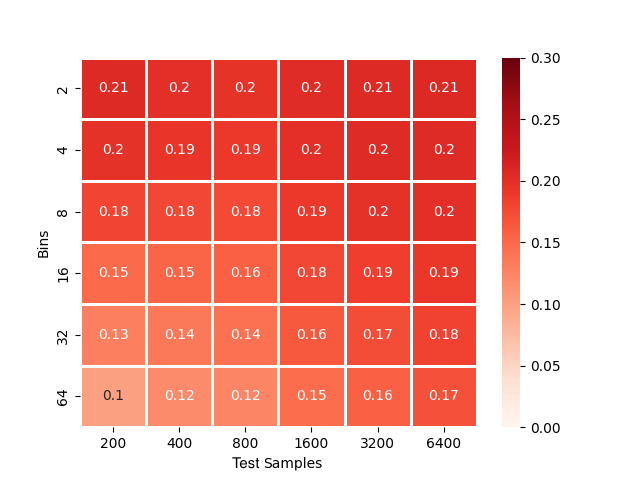}
        \caption{\small{ECE of DT}}
    \end{subfigure}
    \begin{subfigure}{0.48\textwidth}
        \centering
        \includegraphics[trim={0 0 0 1.3cm},clip,width=\textwidth]{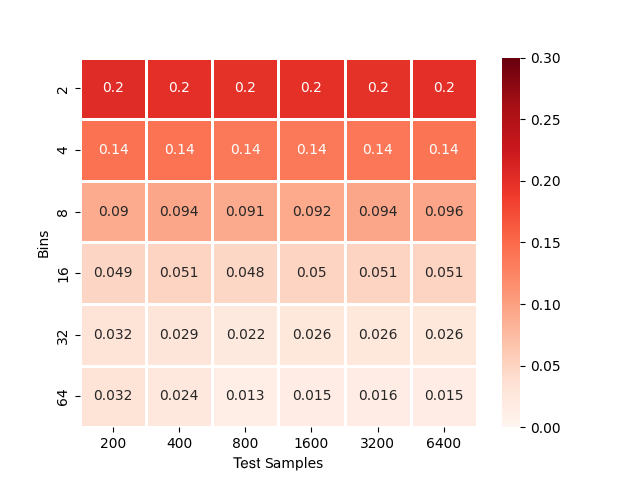}
        \caption{\small{PDE of DT}}
    \end{subfigure}

    \begin{subfigure}{0.48\textwidth}
        \centering
        \includegraphics[trim={0 0 0 1.3cm},clip,width=\textwidth]{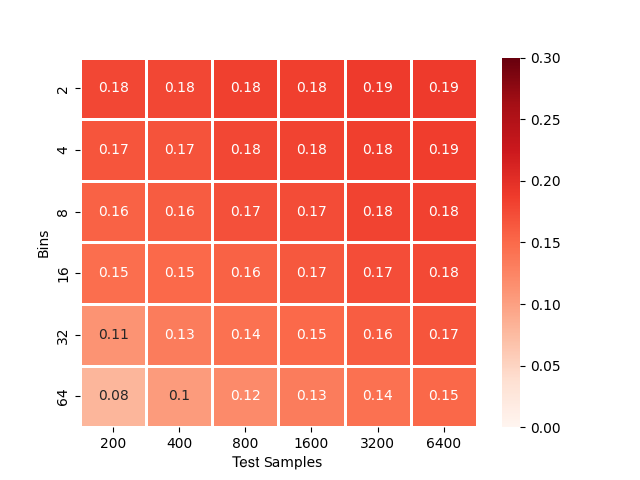}
        \caption{\small{ECE of PCT}}
    \end{subfigure}
    \begin{subfigure}{0.48\textwidth}
        \centering
        \includegraphics[trim={0 0 0 1.3cm},clip,width=\textwidth]{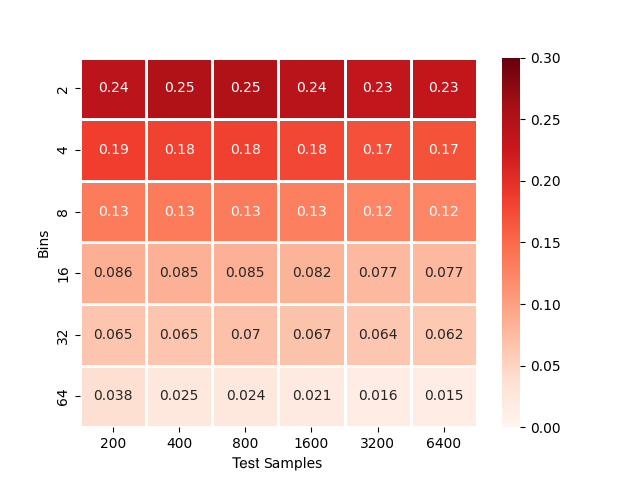}
        \caption{\small{PDE of PCT}}
    \end{subfigure}

    \caption[Calibration metric biases using BFSL binning on complex data]{
     Evaluating the bias of ECE and PDE using BFSL binning on complex data, with the former exhibiting larger bias in most cases, especially with larger numbers of test samples}. 
    \label{fig:ece_pde_bias_bfsl_complexdata}
\end{figure}

\FloatBarrier
\section{Analysing the Calibration-Classification Tradeoff}
\label{sec:calib_class_tradeoff}
To investigate the trade-off between calibration and classification accuracy, we conducted an experiment utilizing decision tree models of varying complexity across multiple datasets. By adjusting the size of the trees, reflected in the number of leaves, we examined how changes in model complexity affect both probability deviation error (PDE) and classification loss.

The experiment begins with an evaluation of decision tree models trained on multiple datasets. These models varied in complexity, as characterized by the size of the tree, i.e., the number of leaves it contains. For each dataset, decision trees of different sizes were trained and subsequently evaluated using two key metrics: Probability deviation error (PDE) and classification loss. We allocated half of each dataset for training the decision tree, with the remaining portion serving as the test set. The datasets we have used are the 36 UCI datasets described in Section \ref{sec:real_datasets} and two synthetic datasets that are also used in the experiments in Section \ref{sec:eval_pde}.

PDE, a calibration metric, was employed to assess the quality of the predicted probabilities produced by the models. It offers a measure of the divergence between the predicted and actual class probabilities. On the other hand, classification loss, a performance metric, was used to evaluate the model's proficiency in correctly classifying instances. It quantifies the discrepancy between predicted and actual class labels.

Additionally, we incorporated the cost-complexity pruned decision tree to contrast its performance against the models trained in this experiment.

The results observed from the experiment were insightful. Figure \ref{fig:calib_vs_class_selected} presents the findings of this experiment for a select group of datasets. Firstly, we found that PDE consistently increased as the size of the decision tree increased across all datasets. This trend indicates that as the decision trees grew in complexity, the calibration quality decreased. In other words, the models' predicted probabilities became less representative of the true class probabilities as the decision trees became larger. The deductions proved in Section \label{cell_merg_and_averaging} supports the results here as we have shown by merging the cells induced by predictor, we will improve the calibration, and by decreasing the size of the decision trees we are doing the same action.

Contrary to the PDE trend, classification loss generally decreased as the size of the decision tree increased. This indicates that more complex trees were typically more successful in their classification tasks. This characteristic persists until the decision tree reaches a size that leads to overfitting on the data. The specific tree size at which this occurs varies across datasets, dependent on their unique characteristics. Figures \ref{fig:calib_vs_class_selected:hand-postures} and \ref{fig:calib_vs_class_selected:phishing} represent this characteristic. However, it's worth noting that this was not an absolute trend, as exceptions have been observed in a few of datasets, a case in point being the dataset represented in Figure \ref{fig:calib_vs_class_selected:vote}.

The performance of cost-complexity pruned decision trees compared to pre-pruned trees trained in this experiment was examined using a combination of PDE and classification loss metrics. 
The outcomes varied considerably across the different datasets. In 11 datasets, such as those represented in Figures \ref{fig:calib_vs_class_selected:phishing} and \ref{fig:calib_vs_class_selected:news-popularity}, the performance of post-pruned trees was observed to be weaker. Conversely, in 9 datasets, like those represented in Figures \ref{fig:calib_vs_class_selected:hand-postures} and \ref{fig:calib_vs_class_selected:synthetic-5}, post-pruned trees exhibited superior performance. In the remaining 18 datasets, such as those shown in Figures \ref{fig:calib_vs_class_selected:vote} and \ref{fig:calib_vs_class_selected:sick}, the performance of the post-pruned trees remained largely unchanged. These findings suggest that while the effectiveness of cost-complexity pruning can vary based on the unique characteristics of each dataset, its overall impact seems to be quite subtle.

In conclusion, our experiment provides evidence supporting the existence of a tradeoff between calibration and classification performance in decision tree models. As decision trees increase in size and complexity, they generally become better at classification as long as they are not overfitted but worse in terms of calibration. It should be noted, however, that this relationship is not universal and can be influenced by dataset-specific factors. Additionally, The performance of a post-pruned tree is nearly identical to that of a decision tree of the same size, and this may be contingent on the distinct characteristics of the data.


\begin{figure*}[ht]
    \centering
    \begin{subfigure}{0.45\textwidth}
        \centering
        \includegraphics[trim={0 0 0 0},clip,width=\textwidth]{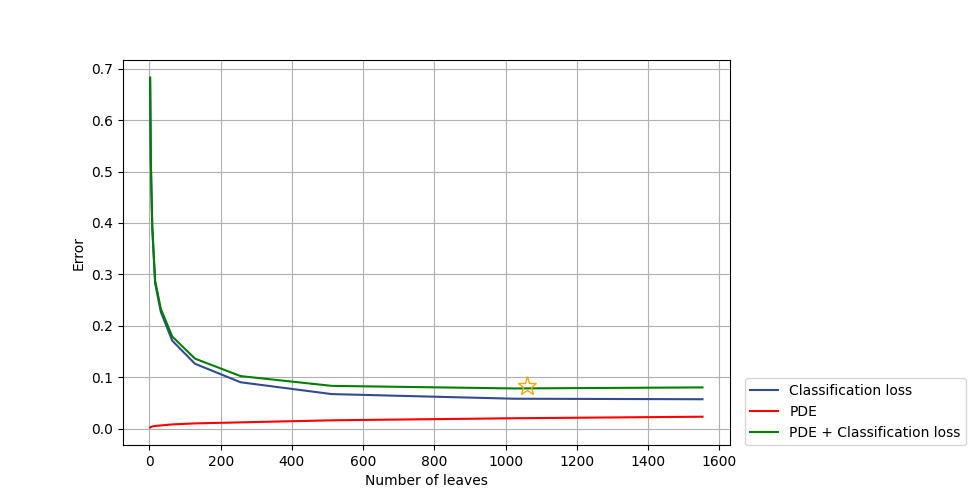}
        \caption{hand-postures\label{fig:calib_vs_class_selected:hand-postures}}
    \end{subfigure}
    \begin{subfigure}{0.45\textwidth}
        \centering
        \includegraphics[trim={0 0 0 0},clip,width=\textwidth]{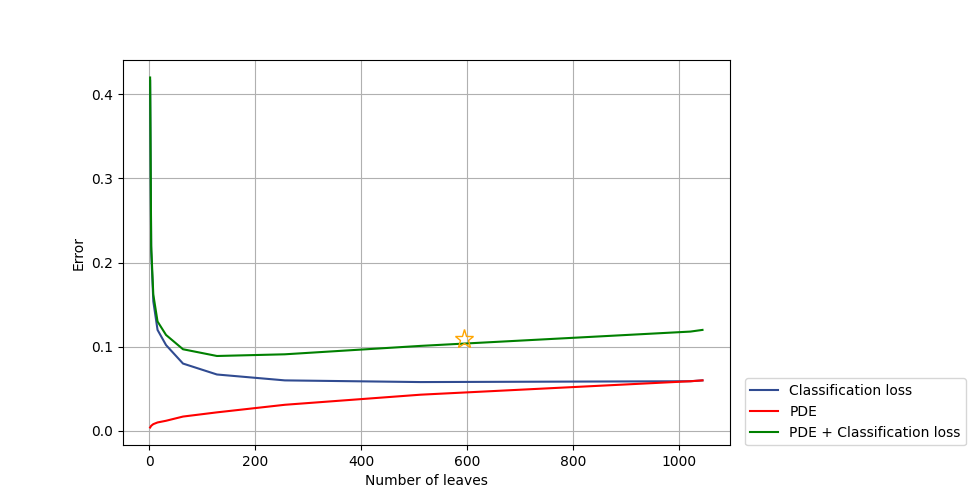}
        \caption{Synthetic-5 generator with 50,000 samples\label{fig:calib_vs_class_selected:synthetic-5}}
    \end{subfigure}
    \vspace{2em}
    
    \begin{subfigure}{0.45\textwidth}
        \centering
        \includegraphics[trim={0 0 0 0},clip,width=\textwidth]{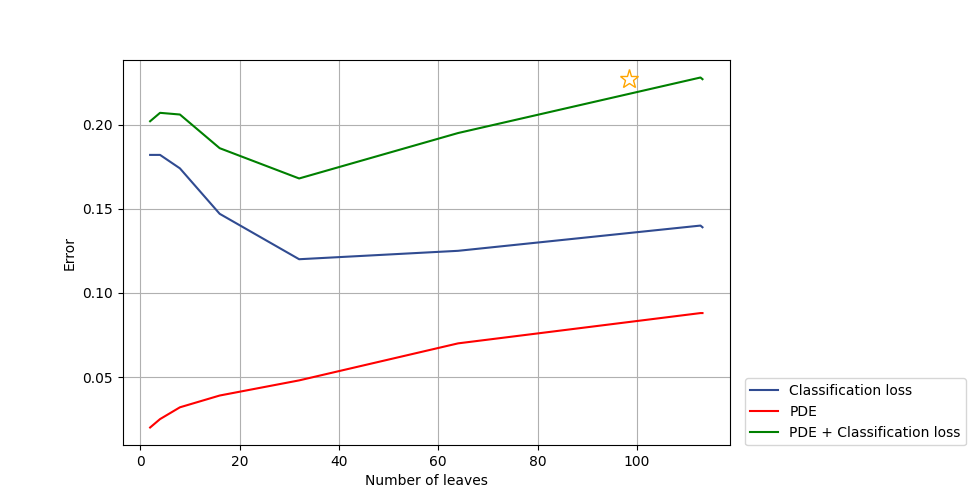}
        \caption{phishing\label{fig:calib_vs_class_selected:phishing}}
    \end{subfigure}
    \begin{subfigure}{0.45\textwidth}
        \centering
        \includegraphics[trim={0 0 0 0},clip,width=\textwidth]{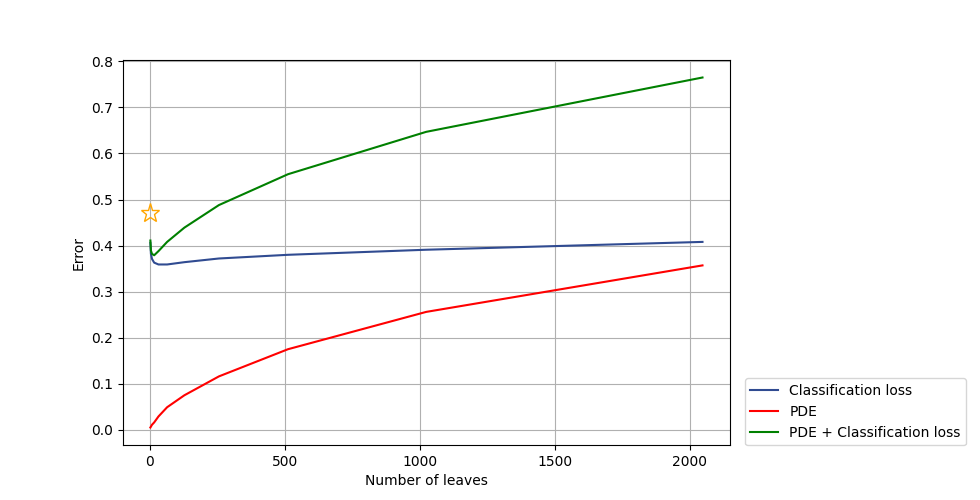}
        \caption{news-popularity\label{fig:calib_vs_class_selected:news-popularity}}
    \end{subfigure}
    \vspace{2em}
    
    \begin{subfigure}{0.45\textwidth}
        \centering
        \includegraphics[trim={0 0 0 0},clip,width=\textwidth]{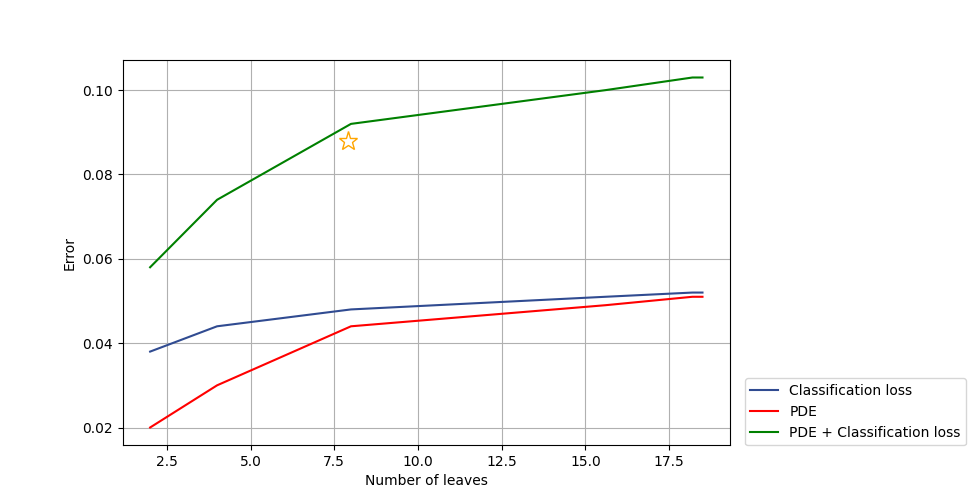}
        \caption{vote\label{fig:calib_vs_class_selected:vote}}
    \end{subfigure}
    \begin{subfigure}{0.45\textwidth}
        \centering
        \includegraphics[trim={0 0 0 0},clip,width=\textwidth]{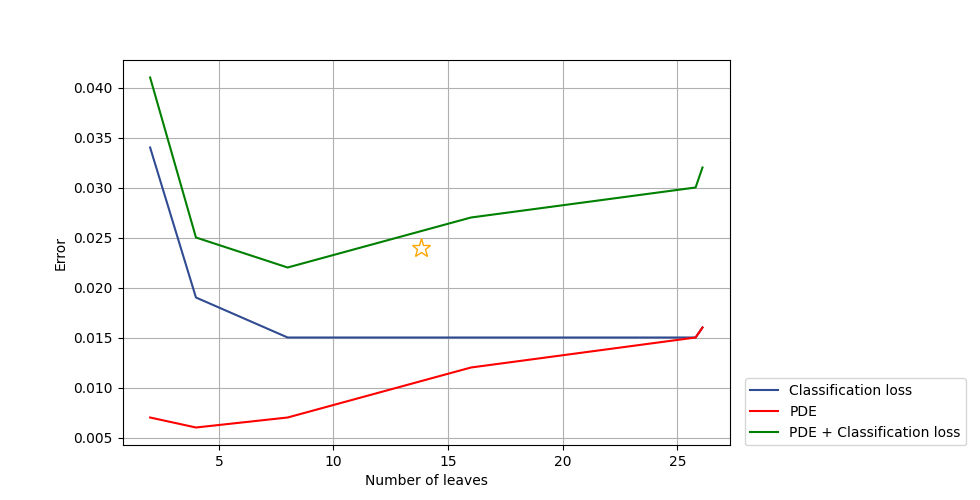}
        \caption{sick\label{fig:calib_vs_class_selected:sick}}
    \end{subfigure}
   
    \caption[Calibration-classification tradeoff over a selected datasets]{
    Evaluating decision trees of varying sizes across multiple datasets using probability deviation error (PDE) and classification loss. The star symbol (\faStarO) denotes the combined total of PDE and classification loss for the decision tree that has undergone cost-complexity post-pruning, also signifying its size.}
    \label{fig:calib_vs_class_selected}
\end{figure*}

\end{document}